\documentclass[11pt]{article}

\usepackage[letterpaper,margin=1in]{geometry}
\usepackage[T1]{fontenc}
\usepackage[utf8]{inputenc}
\usepackage{amsmath,amssymb,amsthm,mathtools}
\usepackage{algorithm}
\usepackage{algpseudocode}
\usepackage{booktabs}
\usepackage{enumitem}
\usepackage{graphicx}
\usepackage{microtype}
\usepackage{natbib}
\usepackage{hyperref}
\hypersetup{
  pdftitle={On the Convergence of Stochastic Low-Rank Adaptation},
  pdfauthor={Ru Wang, Chengchang Liu, John C.S. Lui}
}
\newtheorem{theorem}{Theorem}[section]
\newtheorem{lemma}[theorem]{Lemma}
\newtheorem{proposition}[theorem]{Proposition}

\theoremstyle{definition}
\newtheorem{assumption}{Assumption}
\newtheorem{definition}{Definition}[section]

\newcommand{\R}{\mathbb{R}}
\newcommand{\E}{\mathbb{E}}
\newcommand{\inner}[2]{\left\langle#1,#2\right\rangle_F}
\newcommand{\norm}[1]{\left\|#1\right\|_F}
\newcommand{\Bcal}{\mathcal{B}}
\newcommand{\algname}{\textnormal{\textsc{LoRA-NSGDM}}}
\newcommand{\vralgname}{\textnormal{\textsc{LoRA-STORM}}}

\title{On the Convergence of Stochastic Low-Rank Adaptation}

\author{Ru Wang\textsuperscript{1,*} \qquad
  Chengchang Liu\textsuperscript{2,*,\ensuremath{\dagger}} \qquad
  John C.S. Lui\textsuperscript{1}\\[0.5em]
  \normalsize \textsuperscript{1}The Chinese University of Hong Kong \qquad
  \textsuperscript{2}Westlake University}
\date{}

\begin{document}
\maketitle
\begingroup
\renewcommand{\thefootnote}{\fnsymbol{footnote}}
\footnotetext[1]{Equal contribution.}
\footnotetext[2]{Corresponding author: \texttt{liuchengchang@westlake.edu.cn}.}
\endgroup

\begin{abstract}
Low-rank adaptation (LoRA) minimizes the training loss after adding a
trainable low-rank adapter to a frozen pretrained weight matrix.
Previous analysis of gradient descent for LoRA (LoRA-GD) gives a sufficient complexity of
$\exp(\mathcal{O}(\epsilon^{-2}))$ iterations to find an
$\epsilon$-stationary point in the deterministic setting.
We sharpen this analysis and show that $\mathcal{O}(\epsilon^{-4})$
iterations suffice for the same first-order criterion.
We further study stochastic LoRA under unbiased stochastic gradients with
finite variance.
We propose LoRA-NSGDM, which finds an $\epsilon$-stationary point with
$\mathcal{O}(\epsilon^{-8})$ stochastic oracle calls.
Under an additional mean-square smoothness condition, we propose LoRA-STORM,
which improves the stochastic oracle complexity to $\mathcal{O}(\epsilon^{-6})$.
 \end{abstract}

\section{Introduction}
\label{sec:introduction}
Low-rank adaptation (LoRA) reduces fine-tuning costs by learning a low-rank
update to frozen pretrained weights~\citep{hu2022lora}.
Given a loss \(\mathcal{L}:\R^{m\times n}\to\R\) and frozen weights
\(W_{\mathrm{base}}\in\R^{m\times n}\), we optimize factors
\(B\in\R^{m\times r}\) and \(A\in\R^{r\times n}\) that parameterize the
update as \(sBA\), where \(s\) is a fixed scaling parameter.
Defining \(f:\R^{m\times n}\to\R\) by
\(f(X):=\mathcal{L}(W_{\mathrm{base}}+sX)\) and stacking the factors as
\(V=[B;A^\top]\in\R^{(m+n)\times r}\),
we study the factorized problem
\begin{equation}
  \min_{V\in\R^{(m+n)\times r}} J(V)
  \quad\text{where}\quad J(V):=f(BA).
  \label{obj:LoRA}
\end{equation}

Analyzing the convergence guarantee for LoRA~\eqref{obj:LoRA} remains challenging.
As observed in prior work
\citep{sun2024federated,malinovsky2024rac}, even
when the loss \(f\) has a Lipschitz gradient, the factorized
objective \(J\) generally has no globally Lipschitz gradient in \((B,A)\) because the map \((B,A)\mapsto BA\) is bilinear.  
Consequently, standard convergence theory for smooth nonconvex objectives
does not directly apply to LoRA.
Existing theory studies LoRA's expressive power~\citep{zeng2024expressive},
optimization landscapes and convergence in the NTK regime and
beyond~\citep{jang2024lorantk,kim2025loraconverges}, and gradient-flow
dynamics in matrix factorization~\citep{xu2025loradynamics}.
There are also studies on factor-specific learning rates and gradient
preconditioning in infinite-width limits~\citep{hayou2024loraplus,zhang2024riemannian},
or impose structural conditions on the LoRA factors~\citep{ding2026dynamics}.
These results do not establish a non-asymptotic convergence
guarantee for LoRA gradient descent under only smoothness and lower
boundedness of the original loss.

Recently, \citet{mu2026loraGD} provide a non-asymptotic analysis of LoRA-GD
under smoothness and lower boundedness condition of the loss \(\mathcal{L}\).  They prove
that LoRA-GD converges at a rate of
\[
  \min_{0\le t<T}\norm{\nabla J(V_t)}^2
  =\mathcal{O}\!\left(\frac{1}{\log T}\right).
\] 
This rate yields an exponential complexity bound 
\(\exp(\mathcal{O}(\epsilon^{-2}))\) for finding an $\epsilon$-stationary point, which is not satisfactory.  
This motivates us to ask

\begin{center}
  \setlength{\fboxsep}{8pt}
  \fbox{\parbox{0.86\linewidth}{
    \textit{Q1. Can we improve the iteration complexity of LoRA-GD to a
    polynomial bound?}
  }}
\end{center}

Despite the non-asymptotic analysis for the deterministic setting, the
stochastic counterpart remains open in \citet{mu2026loraGD}, which
identifies the control of a fourth-order noise moment as a central difficulty.
In addition, the bounded variance assumption on the stochastic gradient of $f(\cdot)$ cannot lead to the bounded variance of the stochastic gradient of $J(\cdot)$.
Existing stochastic guarantees either impose additional global regularity and
trajectory conditions~\citep{jiang2024unified} or alter the factor update
\citep{sokolov2025bernoulli}. 
Neither establishes convergence of stochastic LoRA under finite variance condition on stochastic gradient of $f(\cdot)$ alone.  
This motivates our second
question:

\begin{center}
  \setlength{\fboxsep}{8pt}
  \fbox{\parbox{0.86\linewidth}{
    \textit{Q2. Can we develop convergent methods for stochastic LoRA 
    under standard finite variance assumptions?}
  }}
\end{center}

We answer both questions affirmatively through the following contributions.
\begin{itemize}[leftmargin=*]
\item
    {In the deterministic setting, we sharpen the trajectory analysis of
    \citet{mu2026loraGD} and prove that
    \(\mathcal{O}(\epsilon^{-4})\) iterations suffice to
    find \(V\) satisfying \(\norm{\nabla J(V)}\le\epsilon\)
    (Section~\ref{sec:deterministic}).}

  \item
    {In the stochastic setting, we show that LoRA-SGD may fail to
    converge in expectation even when the stochastic gradients of \(f\)
    are unbiased and have bounded variance.
    To address this limitation, we develop \algname{}, which combines joint
    normalization of the two factor directions with exponential moving-average
    momentum~\citep{cutkosky2020momentum}.
    Under smoothness and lower boundedness of \(f\) and the same oracle
    assumptions, we prove that \algname{} returns a point \(V\) satisfying
    \(\E\norm{\nabla J(V)}\le\epsilon\) using
    \(\mathcal{O}(\epsilon^{-8})\) stochastic oracle calls
    (Section~\ref{sec:stochastic}).}

  \item When the stochastic oracle additionally satisfies mean-square
  smoothness, we develop \vralgname{}, which applies the recursive momentum~\citet{cutkosky2019storm} to the normalized stochastic gradient.  
  {We show an improved stochastic oracle complexity of \(\mathcal{O}(\epsilon^{-6})\) to find the $\epsilon$ stationary point
  (Section~\ref{sec:variance-reduction}).}
\end{itemize}

Table~\ref{tab:results-comparison} summarizes the our results with the prior works.

\begin{table}[t]
  \centering
  \caption{
  We compare iteration complexity in the deterministic setting and oracle complexity in the stochastic setting for finding an $\epsilon$-stationary point of $J$ (cf. Definition~\ref{def:factor-stationarity}). MSS denotes mean-square smoothness. }
  \label{tab:results-comparison}
  \setlength{\tabcolsep}{5pt}
  \renewcommand{\arraystretch}{1.18}
  \begin{tabular}{@{}llcc@{}}
    \toprule
    Method & Setting & Complexity & Reference \\
    \midrule
    LoRA-GD & Deterministic
      & \(\exp(\mathcal{O}(\epsilon^{-2}))\)
      & \cite{mu2026loraGD} \\
    LoRA-GD & Deterministic
      & \(\mathcal{O}(\epsilon^{-4})\) & Theorem~\ref{thm:deterministic} \\
    \algname{} & Finite variance
      & \(\mathcal{O}(\epsilon^{-8})\) & Theorem~\ref{thm:stochastic} \\
    \vralgname{} & Finite variance + MSS
      & \(\mathcal{O}(\epsilon^{-6})\) & Theorem~\ref{thm:variance-reduced} \\
    \bottomrule
  \end{tabular}
\end{table}

\paragraph{Organization.}
The remainder of the paper is organized as follows.
Section~\ref{sec:preliminaries} introduces the LoRA factorization, the stationarity measure, and the assumptions on the objective and stochastic oracle.  
Section~\ref{sec:deterministic} establishes the polynomial convergence bound for LoRA-GD. 
Section~\ref{sec:stochastic} explains the
limitation of plain LoRA-SGD and develops \algname{} under finite variance.
Section~\ref{sec:variance-reduction} introduces \vralgname{} under
mean-square smoothness and establishes its improved stochastic oracle
complexity.  
In Section~\ref{sec:experiments}, we provides empirical evaluations of the proposed methods.
Section~\ref{sec:conclusion} concludes the paper.  
We leave the discussion of related work and the detailed proof in the Appendix.

 \section{Preliminaries}
\label{sec:preliminaries}


We fix integers \(m,n,r\ge1\) with \(r\le\min\{m,n\}\).
We use the Frobenius inner product and norm
\[
  \inner{X}{Y}:=\operatorname{tr}(X^\top Y)
  \quad\text{and}\quad
  \norm{X}:=\sqrt{\inner{X}{X}}
\]
for matrices of the same size.  
For a linear map \(A\),
\(\|A\|_{\rm op}\) is the norm induced by the Frobenius norm.

The following proposition show the properties of the gradient of factorized objective.
\begin{proposition}[\citet{mu2026loraGD}]
\label{prop:factor-geometry}
For \(V=\begin{bmatrix}B^{\top}; A\end{bmatrix}^{\top}\), we have
\begin{equation}
  \norm V^2=\norm B^2+\norm A^2
  \quad\text{and}\quad
  \norm{BA}\le\norm B\,\norm A\le\frac12\norm V^2.
  \label{eq:basic-factor-bounds}
\end{equation}
Moreover, the gradient of \(J(\cdot)\) is
\begin{equation}
  \nabla J(V)
  =\begin{bmatrix}
    \nabla f(BA)A^\top\\
    \nabla f(BA)^\top B
  \end{bmatrix}.
  \label{eq:factor-gradient}
\end{equation}
\end{proposition}
We consider to find the approximate first-order stationarity of $J(\cdot)$, which is defined in the following.
\begin{definition}
\label{def:factor-stationarity}
For \(\epsilon\ge0\), a point \(V\) is an
\emph{\(\epsilon\)-first-order stationary point} of \(J(\cdot)\) if
\(\norm{\nabla J(V)}\le\epsilon\).  Equivalently,
\begin{equation*}
  \norm{\nabla J(V)}^2
  =\norm{\nabla f(BA)A^\top}^2
   +\norm{B^\top\nabla f(BA)}^2
  \le\epsilon^2.
\end{equation*}
\end{definition}
For deterministic methods, we report the number of iterations needed to
return \(V\) with \(\norm{\nabla J(V)}\le\epsilon\).  
For stochastic methods, we report stochastic oracle complexity for returning a random
\(\widetilde V\) with \(\E\norm{\nabla J(\widetilde V)}\le\epsilon\).
Fix an arbitrary initialization \(V_0\).  We denote
$
  \Delta:=J(V_0)-f_\star\ge0.
$

We introduce the following standard assumptions on the loss \(f: \R^{m\times n}\to \R\) and the stochastic oracle.
\begin{assumption}
\label{ass:smooth}
We assume that \(f:\R^{m\times n}\to\R\) is continuously differentiable.
There exists a gradient-Lipschitz parameter \(\rho\ge1\) such that, for every
\(W,W'\in\R^{m\times n}\),
\begin{equation*}
  \norm{\nabla f(W)-\nabla f(W')}
  \le\rho\norm{W-W'}.
\end{equation*}
\end{assumption}

\begin{assumption}
\label{ass:lower-bounded}
There exists a finite lower bound \(f_\star\in\R\) such that
$
  f(W)\ge f_\star
$
for every $W\in\R^{m\times n}$.
\end{assumption}

In the stochastic setting, we consider $f(\cdot)$ can be reformulated by
\begin{align}
  \label{eq:stochastic-loss}
   f(W)=\E_\xi[F(W;\xi)].
\end{align}
We impose the standard unbiasedness and finite-variance conditions on $\nabla F(W;\xi)$.
\begin{assumption}
\label{ass:finite-variance}
For every fixed \(W\in\R^{m\times n}\), we assume that the stochastic gradient satisfies
\begin{equation}
  \E_{\xi}[\nabla F(W;\xi)]
  =\nabla f(W)
  \quad\text{and}\quad
  \E_{\xi}[\norm{\nabla F(W;\xi)-\nabla f(W)}^2]
  \le\sigma^2
  \label{eq:stochastic-oracle}
\end{equation}
for some finite \(\sigma\ge0\).
\end{assumption}

We also introduce the following stochastic gradient oracle:
\begin{equation*}
  G(V;\xi):=
  \begin{bmatrix}
    \nabla F(BA;\xi)A^\top\\
    \nabla F(BA;\xi)^\top B
  \end{bmatrix}.
\end{equation*}
The following lemma establishes unbiasedness and bounds the variance of
\(G(V;\xi)\).
\begin{lemma}
\label{lem:factor-oracle-variance}
Under Assumption~\ref{ass:finite-variance}, we have
\begin{equation}
  \E_{\xi}[G(V;\xi)]=\nabla J(V)
  \quad\text{and}\quad
  \E_{\xi}[\norm{G(V;\xi)-\nabla J(V)}^2]
  \le\sigma^2\norm{V}^2.
  \label{eq:factor-oracle-variance}
\end{equation}
\end{lemma}

\section{LoRA gradient descent and its improved iteration complexity}
\label{sec:deterministic}
In this section, we sharpen the convergence analysis of the LoRA gradient
descent (Algorithm~\ref{alg:lora-gd}).
\begin{algorithm}[H]
  \caption{LoRA gradient descent~\citep{mu2026loraGD}}
  \label{alg:lora-gd}
  \begin{algorithmic}[1]
    \setlength{\itemsep}{1pt}
    \Require \(V_0=\begin{bmatrix}B_0\\A_0^\top\end{bmatrix}\),
      horizon \(T\ge1\)
    \For{\(t=0,\ldots,T-1\)}
      \State \(d_t\gets\norm{V_t}^2+\norm{\nabla f(B_tA_t)}\)
      \State \(\eta_t\gets
        \min\{(4\sqrt2\,\rho d_t)^{-1},1\}\)
      \State \(V_{t+1}\gets V_t-\eta_t\begin{bmatrix}
          \nabla f(B_tA_t)A_t^\top\\
          \nabla f(B_tA_t)^\top B_t
        \end{bmatrix}\)
    \EndFor
    \State \(\widehat t\in\arg\min_{0\le t<T}\norm{\nabla J(V_t)}^2\)
    \State \Return \(V_{\widehat t}\)
  \end{algorithmic}
\end{algorithm}

\citet{mu2026loraGD} analyze Algorithm~\ref{alg:lora-gd} through the
equivalent stacked update
\begin{equation*}
  V_{t+1}=V_t-\eta_t\nabla J(V_t),
\end{equation*}
where $V_t : = [B_t^{\top};A_t]^{\top}$. Their descent estimate yields the weighted budget
\[
  \sum_{t=0}^{T-1}\eta_t\norm{\nabla J(V_t)}^2\le 4\Delta,
\]
so the rate is determined by the cumulative stepsize
\(\sum_{t<T}\eta_t\).  Mu and Klabjan bound the cross term in the
factor-norm recursion separately at every iteration using Young's inequality.
This gives \(\norm{V_t}^2=\mathcal{O}(t+1)\), and hence
\(\eta_t=\Omega((t+1)^{-1})\).  The resulting harmonic sum leads to
\(\min_{t<T}\norm{\nabla J(V_t)}^2=\mathcal{O}(1/\log T)\).

We sharpen their analysis by summing the factor-norm recursion before
applying Cauchy--Schwarz and reusing the weighted descent budget.  We obtain the tighter
trajectory bound
\(\norm{V_t}^2=\mathcal{O}(\sqrt{t+1})\).  Moreover, descent and smoothness
give \(\norm{\nabla f(B_tA_t)}\le\sqrt{2\rho\Delta}\).
It follows that the stepsize denominator satisfies
\(d_t=\mathcal{O}(\sqrt{t+1})\), and hence
\(\eta_t=\Omega((t+1)^{-1/2})\).  Therefore,
\(\sum_{t<T}\eta_t=\Omega(\sqrt T)\), and the same weighted budget gives
\(\min_{t<T}\norm{\nabla J(V_t)}=\mathcal{O}(T^{-1/4})\).  
This sharper trajectory estimate yields the following improved iteration
complexity for LoRA-GD.
\begin{theorem}
\label{thm:deterministic}
Under Assumptions~\ref{ass:smooth} and~\ref{ass:lower-bounded}, for any
\(\epsilon\in(0,1]\),
it suffices to run Algorithm~\ref{alg:lora-gd} with 
\begin{equation*}
  T=\mathcal{O}\!\left(
    \left\{1+\Delta^2\left[
      1+\rho\norm{V_0}^2+\rho\Delta
      +\rho^{3/2}\sqrt{\Delta}
    \right]^2\right\}\epsilon^{-4}
  \right).
\end{equation*}
The output then satisfies
\(\norm{\nabla J(V_{\widehat t})}\le\epsilon\).
\end{theorem}
 \section{LoRA normalized stochastic gradient descent with momentum}
\label{sec:stochastic}
In this section, we establish convergence guarantees for stochastic LoRA.  We first show that LoRA-SGD may fail to converge under
only a finite-variance assumption.  We then propose normalized stochastic
gradient descent with momentum (NSGDM) and prove a non-asymptotic convergence
guarantee under the same mild assumptions.

\subsection{LoRA-SGD fails to converge}

We first show why Assumptions~\ref{ass:smooth},
\ref{ass:lower-bounded}, and~\ref{ass:finite-variance} do not suffice for 
stochastic gradient descent, which updates according to
\begin{equation}
  V_{t+1}=V_t-\eta_t
  \begin{bmatrix}
    \nabla F(B_tA_t;\xi_t)A_t^\top\\
    \nabla F(B_tA_t;\xi_t)^\top B_t
  \end{bmatrix},
  \label{eq:plain-lora-sgd}
\end{equation}
where \(\eta_t>0\) may depend on all previous samples and
iterates but is chosen before the oracle query at iteration \(t\).
The following proposition shows that even when the stepsizes are chosen adaptively, the iterates of \eqref{eq:plain-lora-sgd} may diverge in expectation.

\begin{proposition}
\label{prop:plain-lora-sgd-hard-instance}
There are a scalar loss, stochastic oracle, and initialization satisfying
Assumptions~\ref{ass:smooth}, \ref{ass:lower-bounded},
and~\ref{ass:finite-variance} such that,
for every stepsize rule where \(\eta_t>0\) may depend on all previous samples and
iterates before drawing \(\xi_t\), the iterates of
\eqref{eq:plain-lora-sgd} satisfy, for every \(t\ge1\),
\begin{equation}
  \E[J(V_t)]=+\infty
  \quad\text{and}\quad
  \E\norm{\nabla J(V_t)}=+\infty.
  \label{eq:plain-sgd-hard-conclusion}
\end{equation}
\end{proposition}

The difficulty also appears in the analysis based on the descent lemma established in \citet{mu2026loraGD}, which we restate in Lemma \ref{lem:appendix-modified-descent} in Appendix~\ref{app:auxiliary-lemmas} for completeness.
The  descent lemma with
\(U=-\eta_tG(V_t;\xi_t)\) gives
\begin{align}
  J(V_{t+1})- J(V_t)
  &\leq-\eta_t\inner{\nabla J(V_t)}{G(V_t;\xi_t)}
    +\eta_t^2\bigl(\sqrt{2}\rho\norm{V_t}^2+\norm{\nabla f(B_tA_t)}\bigr)
    \norm{G(V_t;\xi_t)}^2 \nonumber\\
  &\quad+\sqrt{2}\rho\eta_t^3\norm{V_t}
    \norm{G(V_t;\xi_t)}^3
    +\frac{\sqrt{2}\rho}{4}\eta_t^4
    \norm{G(V_t;\xi_t)}^4.
  \label{eq:plain-sgd-fourth-moment}
\end{align}
Assumption~\ref{ass:finite-variance} controls only the second moment of the
oracle noise and therefore does not control the cubic and quartic terms in
\eqref{eq:plain-sgd-fourth-moment}, which is regarded as a central challenge by
\citet{mu2026loraGD}.

\subsection{Normalized SGD with Momentum Converges for Stochastic LoRA}

The difficulty above originates from the random length
\(\eta_t\norm{G(V_t;\xi_t)}\) of a plain LoRA-SGD step, which
produces the cubic and quartic stochastic-gradient terms in
\eqref{eq:plain-sgd-fourth-moment}.  We first normalize the update so that
every step has the prescribed length \(\gamma\).  This removes the need to
control higher-order moments of the stochastic gradient.
On the other hand, 
Lemma~\ref{lem:factor-oracle-variance} gives
\[
  \E_{\xi_t}\!\left[
    \norm{G(V_t;\xi_t)-\nabla J(V_t)}^2
  \right]
  \leq \sigma^2\norm{V_t}^2,
\]
which means that $G(V_t;\xi_t)$ may become increasingly noisy as
\(\norm{V_t}\) grows.  
We therefore use momentum to average the stochastic
gradients before normalizing the resulting direction to reduce the variance.  
This leads to the normalized stochastic gradient descent with momentum method~\citep{cutkosky2020momentum} for the stochastic LoRA presented in Algorithm~\ref{alg:lora-nsgdm}.

\begin{algorithm}[H]
  \caption{LoRA-NSGDM}
  \label{alg:lora-nsgdm}
  \begin{algorithmic}[1]
    \setlength{\itemsep}{1pt}
    \Require \(V_0=\begin{bmatrix}B_0\\A_0^\top\end{bmatrix}\), \(T\ge1\)
    \State \(\alpha\gets T^{-1/2}\), \(\gamma\gets T^{-7/8}\), and \(M_{-1}\gets0\)
    \For{\(t=0,\ldots,T-1\)}
      \State \(G(V_t;\xi_t)\gets\begin{bmatrix}
        \nabla F(B_tA_t;\xi_t)A_t^\top\\
        \nabla F(B_tA_t;\xi_t)^\top B_t
      \end{bmatrix}\)
      \State \(M_t\gets(1-\alpha)M_{t-1}+\alpha G(V_t;\xi_t)\)
      \State \(V_{t+1}\gets V_t-\gamma\frac{M_t}{\norm{M_t}}\)
    \EndFor
    \State Draw \(I\) uniformly from \(\{0,\ldots,T-1\}\), independently
      of the oracle samples
    \State \Return \(V_I\)
  \end{algorithmic}
\end{algorithm}

We next give a non-asymptotic analysis of Algorithm~\ref{alg:lora-nsgdm}.
We introduce the deterministic radius and local smoothness scale
\begin{equation}
  R_T=\norm{V_0}+T\gamma
  \quad\text{and}\quad
  L_T=\norm{\nabla f(0)}+\frac{3\rho R_T^2}{2}.
  \label{eq:nsgdm-scales}
\end{equation}
The following lemma shows that the normalization controls the trajectory of the iterates.
\begin{lemma}
\label{lem:nsgdm-local-control}
Suppose Assumptions~\ref{ass:smooth} and~\ref{ass:finite-variance} hold.
Fix any horizon \(T\ge1\), and let
\(\{V_t\}_{t=0}^{T}\) and \(\{M_t\}_{t=-1}^{T-1}\) be generated by
Algorithm~\ref{alg:lora-nsgdm}.  With \(R_T\) and \(L_T\) defined in
\eqref{eq:nsgdm-scales}, all iterates lie in the ball
\begin{equation*}
  V_t\in\Bcal_{R_T}:=\{V:\norm V\le R_T\},
  \qquad 0\le t\le T.
\end{equation*}
Moreover, \(J\) is \(L_T\)-smooth on this ball:
\begin{equation}
  \norm{\nabla J(V)-\nabla J(U)}
  \le L_T\norm{V-U}
  \qquad
  \text{for all }U,V\in\Bcal_{R_T}.
  \label{eq:nsgdm-local-smoothness}
\end{equation}
For every \(0\le t<T\), the factor-gradient oracle satisfies
\begin{equation}
  \E_{\xi_t}[G(V_t;\xi_t)]=\nabla J(V_t)
  \quad\text{and}\quad
  \E_{\xi_t}\norm{G(V_t;\xi_t)-\nabla J(V_t)}^2
  \le\sigma^2R_T^2.
  \label{eq:nsgdm-noise-scale}
\end{equation}
\end{lemma}

We next quantify how momentum balances its initialization error,
the drift of the true gradient, and the averaged oracle noise.

\begin{lemma}
\label{lem:nsgdm-tracking}
Suppose Assumptions~\ref{ass:smooth} and~\ref{ass:finite-variance} hold.
Fix any horizon \(T\ge1\), and let
\(\{V_t\}_{t=0}^{T}\) and \(\{M_t\}_{t=-1}^{T-1}\) be generated by
Algorithm~\ref{alg:lora-nsgdm}.  Then
\begin{equation}
  \frac1T\sum_{t<T}\E\norm{M_t-\nabla J(V_t)}
  \le\frac{\norm{\nabla J(V_0)}}{\alpha T}
  +\frac{L_T\gamma}{\alpha}+\sigma R_T\sqrt\alpha.
  \label{eq:main-tracking}
\end{equation}
\end{lemma}
Then, we show one-step descent lemma of LoRA-NSGDM on the objective $J(\,\cdot\,)$.
\begin{lemma}
\label{lem:nsgdm-descent}
Suppose Assumption~\ref{ass:smooth} holds.
Fix any horizon \(T\ge1\), and let
\(\{V_t\}_{t=0}^{T}\) and \(\{M_t\}_{t=-1}^{T-1}\) be generated by
Algorithm~\ref{alg:lora-nsgdm}.  Then, for every \(0\le t<T\),
\begin{equation}
  J(V_{t+1})
  \le J(V_t)-\gamma\norm{\nabla J(V_t)}
  +2\gamma\norm{M_t-\nabla J(V_t)}
  +\frac{L_T\gamma^2}{2}.
  \label{eq:nsgdm-descent}
\end{equation}
\end{lemma}

Combining the above lemmas gives the convergence rate and oracle
complexity of Algorithm~\ref{alg:lora-nsgdm}.

\begin{theorem}
\label{thm:stochastic}
Under Assumptions~\ref{ass:smooth}, \ref{ass:lower-bounded},
and~\ref{ass:finite-variance}, for any \(\epsilon\in(0,1]\), it suffices to
run Algorithm~\ref{alg:lora-nsgdm} with
\begin{equation}
  T=\mathcal{O}\!\left(
    \left[
      1+\Delta
      +(1+\norm{V_0})\bigl(\norm{\nabla f(0)}+\sigma\bigr)
      +\rho\bigl(\norm{V_0}^3+(1+\norm{V_0})^2\bigr)
    \right]^8\epsilon^{-8}
  \right).
  \label{eq:stochastic-complexity}
\end{equation}
The output then satisfies \(\E\norm{\nabla J(V_I)}\le\epsilon\).  Since each iteration uses one stochastic
oracle evaluation, its stochastic oracle complexity is also given by
\eqref{eq:stochastic-complexity}.
\end{theorem}
 \section{LoRA-STORM with improved stochastic oracle complexity}
\label{sec:variance-reduction}
In this section, we study the convergence of stochatsic LoRA under the mean-square
smoothness conditrion, which is widely used in the variance reduction methods \citep{reddi2016svrg,fang2018spider,cutkosky2019storm,li2021page}.
\begin{assumption}
\label{ass:mss}
There is a finite \(\ell\ge0\) such that, for every
\(W,W'\in\R^{m\times n}\),
\begin{equation*}
  \E_{\xi}\norm{\nabla F(W;\xi)-\nabla F(W';\xi)}^2
  \le\ell^2\norm{W-W'}^2.
\end{equation*}
\end{assumption}

The mean-square smoothness condition yields a corresponding local bound for
the oracle \(G(V;\xi)\).
\begin{proposition}
\label{prop:factor-mss}
Suppose Assumptions~\ref{ass:finite-variance} and~\ref{ass:mss} hold.
For any \(U,V\in\R^{(m+n)\times r}\), we have
\begin{equation}
  \E_{\xi}\norm{G(V;\xi)-G(U;\xi)}^2
  \le\left[3\ell^2R^4
    +4\norm{\nabla f(0)}^2+2\sigma^2\right]\norm{V-U}^2,
  \label{eq:factor-mss}
\end{equation}
where
\(R=\max\{\norm U,\norm V\}\). 
\end{proposition}
\begin{algorithm}[H]
  \caption{LoRA-STORM}
  \label{alg:lora-storm}
  \begin{algorithmic}[1]
    \setlength{\itemsep}{1pt}
    \Require \(V_0=\begin{bmatrix}B_0\\A_0^\top\end{bmatrix}\),
      horizon \(T\ge1\)
    \State \(a\gets T^{-2/3}\) and \(\eta\gets T^{-5/6}\)
    \State \(D_0\gets\begin{bmatrix}
      \nabla F(B_0A_0;\xi_0)A_0^\top\\
      \nabla F(B_0A_0;\xi_0)^\top B_0
    \end{bmatrix}\)
    \State \(V_1\gets V_0-\eta\frac{D_0}{\norm{D_0}}\)
    \For{\(t=1,\ldots,T-1\)}
      \State \(G(V_t;\xi_t)\gets\begin{bmatrix}
        \nabla F(B_tA_t;\xi_t)A_t^\top\\
        \nabla F(B_tA_t;\xi_t)^\top B_t
      \end{bmatrix}\)
      \State \(G(V_{t-1};\xi_t)\gets\begin{bmatrix}
        \nabla F(B_{t-1}A_{t-1};\xi_t)A_{t-1}^\top\\
        \nabla F(B_{t-1}A_{t-1};\xi_t)^\top B_{t-1}
      \end{bmatrix}\)
      \State \(D_t\gets G(V_t;\xi_t)+(1-a)\bigl(D_{t-1}-G(V_{t-1};\xi_t)\bigr)\)
      \State \(V_{t+1}\gets V_t-\eta\frac{D_t}{\norm{D_t}}\)
    \EndFor
    \State Draw \(I\) uniformly from \(\{0,\ldots,T-1\}\), independently
      of the oracle samples
    \State \Return \(V_I\)
  \end{algorithmic}
\end{algorithm}
Proposition~\ref{prop:factor-mss} suggests using variance reduction to improve the stochastic oracle complexity of LoRA-NSGDM.
We apply the STORM method~\cite{cutkosky2019storm} for LoRA in Algorithm~\ref{alg:lora-storm}, where the update direction is equivalent to
\begin{equation}
  D_t= (1-a) D_{t-1} + a G(V_t;\xi_t) + (1-a)(G(V_t;\xi_t)-G(V_{t-1};\xi_{t}))
  \label{eq:stacked-storm-update}
\end{equation}

We compare the update rule~\eqref{eq:stacked-storm-update} with the LoRA-NSGDM update in Algorithm~\ref{alg:lora-nsgdm}.  
The main difference is that LoRA-STORM uses an additional correction term
\((1-a)(G(V_t;\xi_t)-G(V_{t-1};\xi_t))\)
\citep{cutkosky2019storm}.  
This difference estimates the change in the true gradient between two consecutive iterates, and
Proposition~\ref{prop:factor-mss} controls its second moment.  
Since the normalized step gives \(\norm{V_t-V_{t-1}}=\eta\), the correction remains
controlled throughout the run.  
The following theorem shows that LoRA-STORM finds an \(\epsilon\)-stationary
point with \(\mathcal{O}(\epsilon^{-6})\) stochastic oracle complexity.

\begin{theorem}
\label{thm:variance-reduced}
Under Assumptions~\ref{ass:lower-bounded}, \ref{ass:finite-variance},
and~\ref{ass:mss}, for any
\(\epsilon\in(0,1]\), it suffices to run
Algorithm~\ref{alg:lora-storm} with
\begin{equation}
  T=\mathcal{O}\!\left(
    \left[
      1+\Delta
      +(1+\norm{V_0})\bigl(\norm{\nabla f(0)}+\sigma\bigr)
      +\ell(1+\norm{V_0})^2
    \right]^6\epsilon^{-6}
  \right).
  \label{eq:variance-reduced-complexity}
\end{equation}
The output then satisfies \(\E\norm{\nabla J(V_I)}\le\epsilon\).  
Algorithm~\ref{alg:lora-storm} uses \(1+2(T-1)=2T-1\) stochastic oracle
evaluations, so its stochastic oracle complexity is also given by
\eqref{eq:variance-reduced-complexity}.
\end{theorem}

\section{Experiments}
\label{sec:experiments}

We conduct experiments to verify the practical implications of our theory. In
particular, we compare the convergence of \algname{} and \vralgname{} with
LoRA-GD using the adaptive stepsizes of \citet{mu2026loraGD}. We consider three
tasks. The first trains logistic regression on pretrained image representations
on CIFAR-10. The second trains a LoRA-parameterized ResNet-18 on CIFAR-10. The
third fine-tunes TinyLlama-1.1B on Alpaca.
We evaluate \vralgname{} only in the logistic-regression experiment. Although
\vralgname{} offers improved stochastic-oracle complexity in theory, each
iteration requires two stochastic-oracle evaluations. This additional cost
makes \vralgname{} less attractive for the larger-scale ResNet-18 and TinyLlama
experiments, where we compare \algname{} with the LoRA-GD baselines.

Unless stated otherwise, we follow the experimental protocol of
\citet{mu2026loraGD}, which includes the datasets, architectures, LoRA ranks,
factor initializations, batch sizes, and training horizons. The baselines are
LoRA-GD with their \(\eta^{\mathrm{adapt}}\),
\(\eta^{\mathrm{adapt2}}\), and \(\eta^{\mathrm{norm}}\) stepsizes. Following \citet{mu2026loraGD}, we omit
\(\eta^{\mathrm{adapt}}\) from the ResNet-18 and TinyLlama experiments because
computing the required product-space gradient incurs substantial memory
overhead for large models. Within each experiment, all methods use the same data split, initial predictor, and
minibatch order. Appendix~\ref{app:experimental-details} gives the complete
experimental settings and hyperparameters.

\begin{figure}[!ht]
  \centering
  \includegraphics[width=0.8\textwidth]{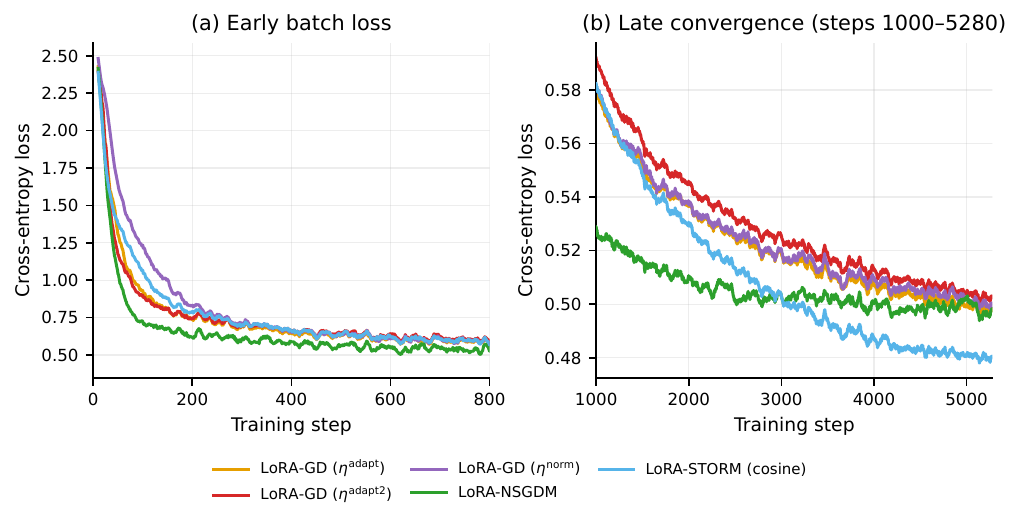}
  \caption{Training a rank-4 LoRA logistic-regression classifier on CIFAR-10
  representations.}
  \label{fig:feature-convergence}
\end{figure}
\paragraph{Logistic Regression.}
We first consider logistic regression on CIFAR-10 representations extracted by
an ImageNet-pretrained \cite{deng2009imagenet} ResNet-18 model \cite{he2016deep}. Figure~\ref{fig:feature-convergence} presents
the convergence comparison. Each CIFAR-10 image is passed once through the pretrained network. We use the resulting 512-dimensional penultimate-layer representation as the input to the linear classifier. These representations are computed before training and
are shared by all methods. The classifier weight is \(W_0+BA\). The base matrix \(W_0\) is fixed, and
only the rank-4 factors \(A\) and \(B\) are trained. We initialize \(A\) with PyTorch's Kaiming-uniform rule, and set \(B=0\).
Training
uses cross-entropy, shuffled minibatches of size 512, no weight decay, and 88
steps per epoch. Panels (a) and (b) of Figure~\ref{fig:feature-convergence} show
the first 60 epochs. In Figure~\ref{fig:feature-convergence}a and \ref{fig:feature-convergence}b, \algname{} exhibits faster convergence than the LoRA-GD algorithms. Over the final 2,000 updates shown in Figure~\ref{fig:feature-convergence}(b), \vralgname{} attains the lowest mean loss.

\begin{figure}[!ht]
  \centering
  \includegraphics[width=0.8\textwidth]{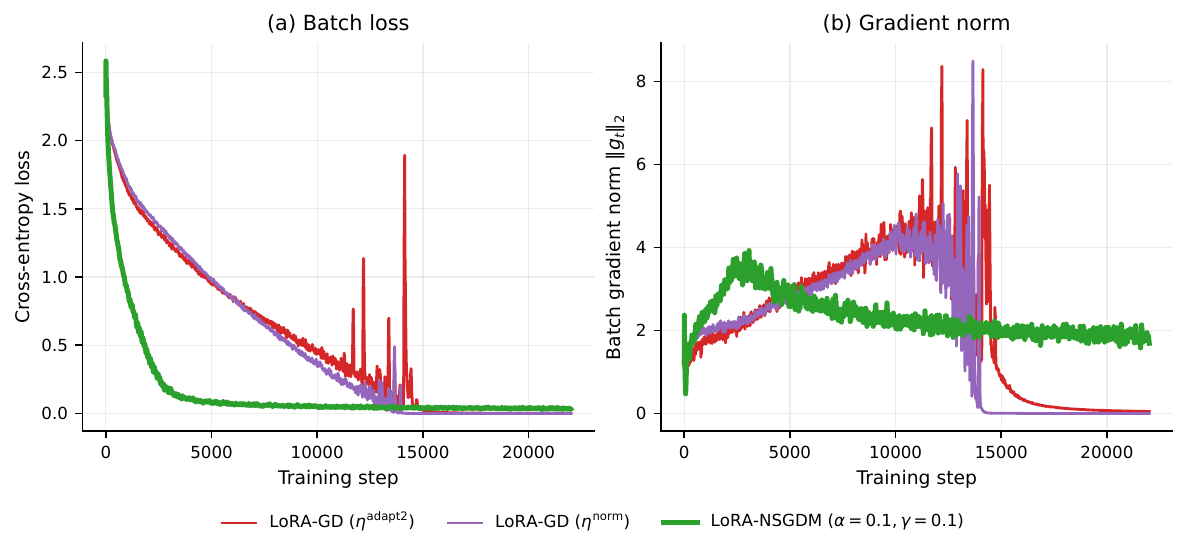}
  \caption{Training ResNet-18 on CIFAR-10 with rank-20 LoRA.}
  \label{fig:resnet-convergence}
\end{figure}

\paragraph{ResNet-18.}
We next train a ResNet-18 directly on CIFAR-10. Rank-20 LoRA factors are applied to
the convolution layers. The output layer is trained without a rank constraint. Batch normalization layers are disabled following the settings in \citet{mu2026loraGD}. We use
minibatches of size 512, cross-entropy, and no weight decay. Each method is
trained for 22,000 steps. Figure~\ref{fig:resnet-convergence}
shows the results over the first 250 epochs. \algname{} converges significantly faster and more steadily than LoRA-GD algorithms.

\begin{figure}[!ht]
  \centering
  \includegraphics[width=0.8\textwidth]{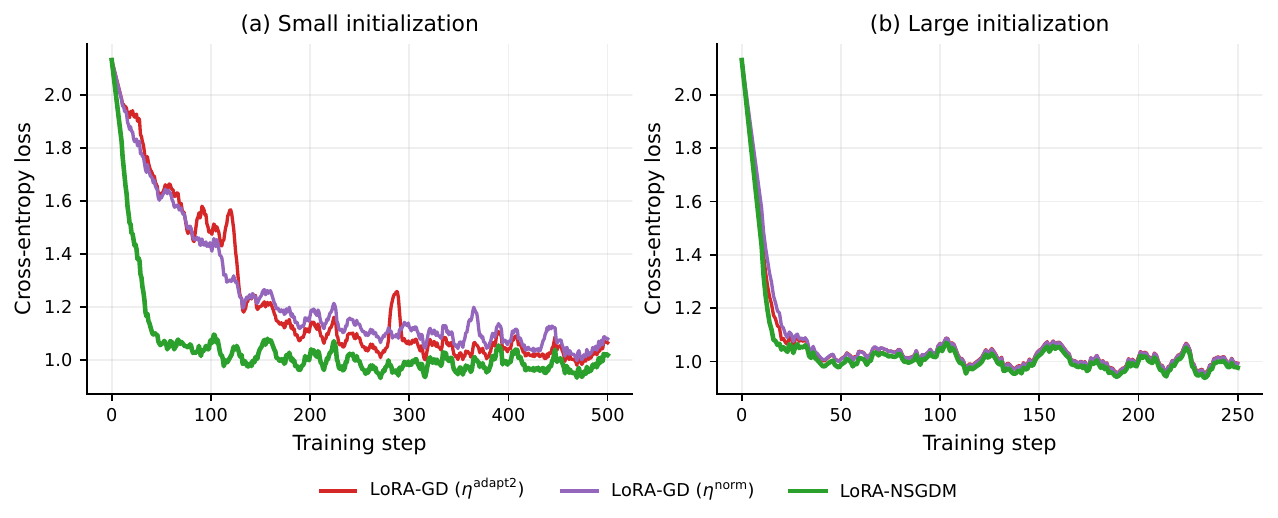}
  \caption{Fine-tuning TinyLlama-1.1B on Alpaca with rank-32 LoRA under
  (a) small and (b) large factor initialization.}
  \label{fig:tinyllama-initialization}
\end{figure}

\paragraph{TinyLlama-1.1B.} Finally, we evaluate the methods on TinyLlama-1.1B instruction fine-tuning.
Figure~\ref{fig:tinyllama-initialization} compares small and large factor
initialization. We fine-tune TinyLlama-1.1B-Chat-v1.0 \citep{zhang2024tinyllama} on Alpaca
\citep{taori2023alpaca}. Each example contains an instruction, an optional
input, and a response. Sequences are padded or truncated to 512 tokens, and
padding tokens are excluded from the causal-language-model loss. We reserve
2,000 examples for validation. Each step uses one minibatch of 16 examples,
with no gradient accumulation. Rank-32 LoRA factors are inserted into the query, key, value, and output
projections of every attention block. The base-model weights and biases remain
fixed.
We set \(B=0\) and sample \(A\sim\mathcal N(0,\sigma^2)\). Small initialization
uses \(\sigma=10^{-3}\) and 500 steps. Large initialization uses
\(\sigma=1/r=0.03125\) and 250 steps. Weight decay is zero.

With the small initialization shown in Figure~\ref{fig:tinyllama-initialization}a, \algname{} maintains a faster convergence rate compared to the LoRA-GD algorithms. On the other hand, in Figure~\ref{fig:tinyllama-initialization}b, large initialization accelerates all three methods and reduces the separation. This
behavior is consistent with our analysis: normalized momentum is most useful
during the noisy initial phase induced by small factors. With large
initialization, the trajectories become similar much earlier.

 \section{Conclusion}
\label{sec:conclusion}

We established finite-time convergence guarantees for 
LoRA optimization.  
For deterministic LoRA-GD, we provide a refined analysis that 
improves the previously exponential sufficient complexity to
\(\mathcal{O}(\epsilon^{-4})\) iterations for finding $\epsilon$-stationary point of $J(\cdot)$.  
In the stochastic setting, we showed that plain LoRA-SGD
can fail under finite variance alone and introduced \algname{}, which achieve \(\mathcal{O}(\epsilon^{-8})\) stochastic
oracle complexity for \(\E\norm{\nabla J(V)}\le\epsilon\).  
When the stochastic oracle satisfies mean-square smoothness, \vralgname{} improves the stochastic
oracle complexity to \(\mathcal{O}(\epsilon^{-6})\).  
Together, these results give polynomial deterministic and stochastic convergence guarantees for
LoRA under the minimal oracle assumptions.

\begingroup
\sloppy
\bibliographystyle{plainnat}
\bibliography{references}

@inproceedings{aghajanyan2021intrinsic,
  author       = {Aghajanyan, Armen and Gupta, Sonal and Zettlemoyer, Luke},
  title        = {Intrinsic Dimensionality Explains the Effectiveness of
                  Language Model Fine-Tuning},
  booktitle    = {ACL-IJCNLP},
  year         = {2021},
}

@inproceedings{hu2022lora,
  title        = {{LoRA}: Low-Rank Adaptation of Large Language Models},
  author       = {Hu, Edward J. and Shen, Yelong and Wallis, Phillip and
                  Allen-Zhu, Zeyuan and Li, Yuanzhi and Wang, Shean and
                  Wang, Lu and Chen, Weizhu},
  booktitle    = {ICLR},
  year         = {2022},
}

@inproceedings{sun2024federated,
  author       = {Sun, Youbang and Li, Zitao and Li, Yaliang and Ding, Bolin},
  title        = {Improving {LoRA} in Privacy-Preserving Federated Learning},
  booktitle    = {ICLR},
  year         = {2024},
}

@inproceedings{ghiasvand2025decentralized,
  author       = {Ghiasvand, Sajjad and Alizadeh, Mahnoosh and
                  Pedarsani, Ramtin},
  title        = {Decentralized Low-Rank Fine-Tuning of Large Language Models},
  booktitle    = {REALM Workshop},
  year         = {2025},
}

@article{malinovsky2024rac,
  author       = {Malinovsky, Grigory and Michieli, Umberto and
                  Hammoud, Hasan Abed Al Kader and Ceritli, Taha and
                  Elesedy, Hayder and Ozay, Mete and
                  Richt{\'a}rik, Peter},
  title        = {Randomized Asymmetric Chain of {LoRA}: The First Meaningful
                  Theoretical Framework for Low-Rank Adaptation},
  journal      = {arXiv preprint arXiv:2410.08305},
  year         = {2024},
  doi          = {10.48550/arXiv.2410.08305}
}

@inproceedings{dettmers2023qlora,
  author       = {Dettmers, Tim and Pagnoni, Artidoro and Holtzman, Ari and
                  Zettlemoyer, Luke},
  title        = {{QLoRA}: Efficient Finetuning of Quantized {LLM}s},
  booktitle    = {NeurIPS},
  year         = {2023},
}

@inproceedings{hayou2024loraplus,
  author       = {Hayou, Soufiane and Ghosh, Nikhil and Yu, Bin},
  title        = {{LoRA}+: Efficient Low Rank Adaptation of Large Models},
  booktitle    = {ICML},
  year         = {2024},
}

@inproceedings{meng2024pissa,
  author       = {Meng, Fanxu and Wang, Zhaohui and Zhang, Muhan},
  title        = {{PiSSA}: Principal Singular Values and Singular Vectors
                  Adaptation of Large Language Models},
  booktitle    = {NeurIPS},
  year         = {2024},
}

@inproceedings{wang2024loraga,
  author       = {Wang, Shaowen and Yu, Linxi and Li, Jian},
  title        = {{LoRA-GA}: Low-Rank Adaptation with Gradient Approximation},
  booktitle    = {NeurIPS},
  year         = {2024},
}

@inproceedings{jang2024lorantk,
  author       = {Jang, Uijeong and Lee, Jason D. and Ryu, Ernest K.},
  title        = {{LoRA} Training in the {NTK} Regime has No Spurious Local
                  Minima},
  booktitle    = {ICML},
  year         = {2024},
}

@inproceedings{zeng2024expressive,
  author       = {Zeng, Yuchen and Lee, Kangwook},
  title        = {The Expressive Power of Low-Rank Adaptation},
  booktitle    = {ICLR},
  year         = {2024},
}

@inproceedings{kim2025loraconverges,
  author       = {Kim, Junsu and Kim, Jaeyeon and Ryu, Ernest K.},
  title        = {{LoRA} Training Provably Converges to a Low-Rank Global
                  Minimum or It Fails Loudly (But It Probably Won't Fail)},
  booktitle    = {ICML},
  year         = {2025},
}

@inproceedings{xu2025loradynamics,
  author       = {Xu, Ziqing and Min, Hancheng and MacDonald, Lachlan Ewen and
                  Luo, Jinqi and Tarmoun, Salma and Mallada, Enrique and
                  Vidal, Rene},
  title        = {Understanding the Learning Dynamics of {LoRA}: A Gradient
                  Flow Perspective on Low-Rank Adaptation in Matrix
                  Factorization},
  booktitle    = {AISTATS},
  year         = {2025},
}

@inproceedings{zhang2024riemannian,
  author       = {Zhang, Fangzhao and Pilanci, Mert},
  title        = {Riemannian Preconditioned {LoRA} for Fine-Tuning Foundation
                  Models},
  booktitle    = {ICML},
  year         = {2024},
}

@inproceedings{zhang2025loraone,
  author       = {Zhang, Yuanhe and Liu, Fanghui and Chen, Yudong},
  title        = {{LoRA}-One: One-Step Full Gradient Could Suffice for
                  Fine-Tuning Large Language Models, Provably and Efficiently},
  booktitle    = {ICML},
  year         = {2025},
}

@inproceedings{jiang2024unified,
  author       = {Jiang, Zhanhong and Saadati, Nastaran and Balu, Aditya and
                  Pham, Minh and Waite, Joshua R. and Saleem, Nasla and
                  Hegde, Chinmay and Sarkar, Soumik},
  title        = {A Unified Convergence Theory for Large Language Model
                  Efficient Fine-Tuning},
  booktitle    = {NeurIPS OPT Workshop},
  year         = {2024},
}

@inproceedings{ding2026dynamics,
  author       = {Ding, Shu and Peng, Yang and Zhou, Hangan and Lu, Xinyu and
                  Chen, Shangwei and Huang, Junhua and Yuan, Mingxuan and
                  Wang, Wei},
  title        = {Towards Understanding the Dynamics of Low-Rank Adaptation},
  booktitle    = {ICML},
  year         = {2026},
}

@inproceedings{mu2026loraGD,
  author       = {Mu, Siqiao and Klabjan, Diego},
  title        = {On the Convergence Rate of {LoRA} Gradient Descent},
  booktitle    = {ICML},
  year         = {2026},
}

@article{sokolov2025bernoulli,
  author       = {Sokolov, Igor and Sadiev, Abdurakhmon and Demidovich, Yury and
                  Al-Qahtani, Fawaz S. and Richt{\'a}rik, Peter},
  title        = {Bernoulli-{LoRA}: A Theoretical Framework for Randomized
                  Low-Rank Adaptation},
  journal      = {arXiv preprint arXiv:2508.03820},
  year         = {2025},
  doi          = {10.48550/arXiv.2508.03820}
}

@inproceedings{cutkosky2020momentum,
  author       = {Cutkosky, Ashok and Mehta, Harsh},
  title        = {Momentum Improves Normalized {SGD}},
  booktitle    = {ICML},
  year         = {2020},
}

@inproceedings{johnson2013svrg,
  author       = {Johnson, Rie and Zhang, Tong},
  title        = {Accelerating Stochastic Gradient Descent using
                  Predictive Variance Reduction},
  booktitle    = {NeurIPS},
  year         = {2013},
}

@inproceedings{nguyen2017sarah,
  author       = {Nguyen, Lam M. and Liu, Jie and Scheinberg, Katya and
                  Tak{\'a}{\v{c}}, Martin},
  title        = {{SARAH}: A Novel Method for Machine Learning Problems
                  Using Stochastic Recursive Gradient},
  booktitle    = {ICML},
  year         = {2017},
}

@inproceedings{reddi2016svrg,
  author       = {Reddi, Sashank J. and Hefny, Ahmed and Sra, Suvrit and
                  P{\'o}czos, Barnab{\'a}s and Smola, Alex},
  title        = {Stochastic Variance Reduction for Nonconvex Optimization},
  booktitle    = {ICML},
  year         = {2016},
}

@inproceedings{fang2018spider,
  author       = {Fang, Cong and Li, Chris Junchi and Lin, Zhouchen and
                  Zhang, Tong},
  title        = {{SPIDER}: Near-Optimal Non-Convex Optimization via
                  Stochastic Path-Integrated Differential Estimator},
  booktitle    = {NeurIPS},
  year         = {2018},
}

@inproceedings{li2021page,
  author       = {Li, Zhize and Bao, Hongyan and Zhang, Xiangliang and
                  Richt{\'a}rik, Peter},
  title        = {{PAGE}: A Simple and Optimal Probabilistic Gradient
                  Estimator for Nonconvex Optimization},
  booktitle    = {ICML},
  year         = {2021},
}

@article{pham2020proxsarah,
  author       = {Pham, Nhan H. and Nguyen, Lam M. and Phan, Dzung T. and
                  Tran-Dinh, Quoc},
  title        = {{ProxSARAH}: An Efficient Algorithmic Framework for
                  Stochastic Composite Nonconvex Optimization},
  journal      = {Journal of Machine Learning Research},
  volume       = {21},
  number       = {110},
  pages        = {1--48},
  year         = {2020},
}

@inproceedings{cutkosky2019storm,
  author       = {Cutkosky, Ashok and Orabona, Francesco},
  title        = {Momentum-Based Variance Reduction in Non-Convex {SGD}},
  booktitle    = {NeurIPS},
  year         = {2019},
}

@article{xu2023pstorm,
  author       = {Xu, Yangyang and Xu, Yibo},
  title        = {Momentum-Based Variance-Reduced Proximal Stochastic Gradient
                  Method for Composite Nonconvex Stochastic Optimization},
  journal      = {Journal of Optimization Theory and Applications},
  volume       = {196},
  number       = {1},
  pages        = {266--297},
  year         = {2023},
  doi          = {10.1007/s10957-022-02132-w},
}

@article{zhang2024tinyllama,
  author       = {Zhang, Peiyuan and Zeng, Guangtao and Wang, Tianduo and
                  Lu, Wei},
  title        = {{TinyLlama}: An Open-Source Small Language Model},
  journal      = {arXiv preprint arXiv:2401.02385},
  year         = {2024},
  doi          = {10.48550/arXiv.2401.02385}
}

@misc{taori2023alpaca,
  author       = {Taori, Rohan and Gulrajani, Ishaan and Zhang, Tianyi and
                  Dubois, Yann and Li, Xuechen and Guestrin, Carlos and
                  Liang, Percy and Hashimoto, Tatsunori B.},
  title        = {Stanford Alpaca: An Instruction-Following {LLaMA} Model},
  year         = {2023},
}

@inproceedings{deng2009imagenet,
  author       = {Deng, Jia and Dong, Wei and Socher, Richard and Li, Li-Jia and
                  Li, Kai and Fei-Fei, Li},
  title        = {{ImageNet}: A Large-Scale Hierarchical Image Database},
  booktitle    = {CVPR},
  year         = {2009},
}

@inproceedings{he2016deep,
  author    = {He, Kaiming and Zhang, Xiangyu and Ren, Shaoqing and Sun, Jian},
  title     = {Deep Residual Learning for Image Recognition},
  booktitle    = {CVPR},
  year      = {2016},
}
\endgroup

\clearpage
\appendix
\section{Related Work}
\label{app:related-work}

\paragraph{Variants of LoRA.}
Adapting pretrained language models to downstream tasks can be costly when all model parameters are updated.
Low-rank adaptation (LoRA) addresses this problem by freezing the pretrained weights and learning low-rank updates, substantially reducing the number of trainable parameters and the memory required for optimizer states~\citep{hu2022lora}.
This approach is consistent with evidence that pretrained language models can be effectively fine-tuned in low-dimensional parameter subspaces~\citep{aghajanyan2021intrinsic}.
Subsequent variants of LoRA target memory efficiency through quantization~\citep{dettmers2023qlora} and optimization efficiency through
factor-specific learning rates~\citep{hayou2024loraplus}.
Initialization strategies that incorporate information from pretrained weights or full fine-tuning gradients can further improve the performance of LoRA~\citep{meng2024pissa,wang2024loraga}.
LoRA can also be extended to privacy-preserving federated learning and
decentralized fine-tuning~\citep{sun2024federated,ghiasvand2025decentralized}.

\paragraph{Convergence analysis of LoRA.}
Existing theory studies LoRA's optimization landscape and convergence
under several sets of assumptions.
In the neural tangent kernel (NTK) regime,
\citet{jang2024lorantk} establish the existence of low-rank optimal
solutions and show that sufficiently large adapter ranks eliminate
spurious local minima after a generic small perturbation of the objective.
Beyond linearized models, \citet{kim2025loraconverges} characterize
second-order stationary points under restricted strong convexity and
smoothness when a low-rank global optimum exists, and show that spurious
solutions must have sufficiently large rank and magnitude.
For matrix factorization, \citet{xu2025loradynamics} analyze
continuous-time gradient flow and characterize the dependence of the
limiting error on the initialization scale and singular-subspace alignment.
Initialization also plays a central role in LoRA-One:
\citet{zhang2025loraone} establish linear convergence from spectral
initialization for linear models and for preconditioned updates in a
nonlinear model, under assumptions on the data distribution and the
low-rank change between tasks.
For stochastic LoRA, \citet{jiang2024unified} obtain
stationary-point guarantees under bounded factor singular values,
bounded full gradients, and bounded stochastic-gradient variance in the
factor variables.
Another line of work modifies the factor updates.
RAC-LoRA~\citep{malinovsky2024rac} and
Bernoulli-LoRA~\citep{sokolov2025bernoulli} sample and freeze one factor
while optimizing the other, then accumulate the resulting low-rank
update into the weight matrix.
These frameworks establish stationarity guarantees for smooth nonconvex
losses and linear convergence of their gradient-descent variants under
the Polyak--{\L}ojasiewicz condition.
For simultaneous gradient updates of both factors,
\citet{mu2026loraGD} assume only smoothness and lower boundedness of the
original loss and prove an \(\mathcal{O}(1/\log T)\) bound on
\(\min_{0\le t<T}\norm{\nabla J(V_t)}^2\), with a stepsize adapted to the
current factors and gradient.
Their analysis handles the lack of global smoothness after factorization
without assuming uniformly bounded factors.
For a broader review of LoRA convergence theory, we refer to
\citet[Section~2]{mu2026loraGD}.

\paragraph{Variance reduction for nonconvex optimization.}
Variance-reduction methods improve the oracle complexity of nonconvex
stochastic optimization.
SVRG~\citep{johnson2013svrg} and SARAH~\citep{nguyen2017sarah}
were introduced for smooth convex finite-sum optimization, with linear
convergence guarantees under strong convexity.
For smooth nonconvex finite-sum objectives, \citet{reddi2016svrg}
establish non-asymptotic stationarity guarantees for SVRG.
Recursive gradient estimators subsequently yield sharper guarantees,
including SPIDER~\citep{fang2018spider} and PAGE~\citep{li2021page}.
For lower-bounded online objectives with global mean-square smoothness
and uniformly bounded stochastic-gradient variance, these two methods
find a point whose expected gradient norm is at most \(\epsilon\) using
\(\mathcal{O}(\epsilon^{-3})\) stochastic gradient evaluations.
Their recursive updates estimate gradient changes by evaluating
stochastic gradients at two successive iterates using the same sample.
STORM incorporates such a correction into momentum and avoids periodic
large-batch gradient computations~\citep{cutkosky2019storm}.
ProxSARAH and PStorm extend recursive variance reduction to composite
nonconvex problems, with guarantees stated in terms of the proximal
gradient mapping~\citep{pham2020proxsarah,xu2023pstorm}.
Applying these stochastic guarantees requires smoothness and variance
conditions in the variables being optimized.
For LoRA, global smoothness of \(f\) does not imply global smoothness of
\(J\), and the factor-gradient variance bound grows with
\(\norm V^2\) (Lemma~\ref{lem:factor-oracle-variance}).
The mean-square smoothness bound for the factor oracle also depends
on the factor norms (Proposition~\ref{prop:factor-mss}).
Our analysis controls these quantities along the normalized factor
trajectory to establish convergence of \vralgname{}.

\section{Auxiliary Lemmas}
\label{app:auxiliary-lemmas}
\begin{lemma}[{\normalfont\citealp[Lemma~3.3]{mu2026loraGD}}]
\label{lem:appendix-modified-descent}
Under Assumption~\ref{ass:smooth}, for every
\(V,U\in\R^{(m+n)\times r}\), we have
\begin{align*}
  &J(V+U)\le{}J(V)+\inner{\nabla J(V)}{U}
  +\sqrt2\,\rho\norm{U}^2\norm{V}^2
  +\sqrt2\,\rho\norm{U}^3\norm{V}\\
  &+\frac{\sqrt2\,\rho}{4}\norm{U}^4
  +\norm{\nabla f(BA)}\norm{U}^2.
\end{align*}
\end{lemma}

\begin{lemma}[{\normalfont\citealp[Lemma~3.4]{mu2026loraGD}}]
\label{lem:appendix-one-step}
Under Assumption~\ref{ass:smooth}, let
\(V_+=V-\eta\nabla J(V)\), where the stepsize satisfies
\[
  0<\eta\le\min\left\{
  \frac{1}{4\sqrt2\,\rho(\norm V^2+\norm{\nabla f(BA)})},1
  \right\}.
\]
Then we have
\begin{equation*}
  J(V_+)\le J(V)-\frac{\eta}{4}\norm{\nabla J(V)}^2.
\end{equation*}
\end{lemma}

\begin{lemma}[{\normalfont\citealp[Eq.~(23)]{mu2026loraGD}}]
\label{lem:appendix-gradient-bound}
If Assumptions~\ref{ass:smooth} and~\ref{ass:lower-bounded} hold, then we have
\[
  \norm{\nabla f(W)}^2\le2\rho(f(W)-f_\star).
\]
\end{lemma}

\section{Proof of Section~\ref{sec:preliminaries}}
\label{app:preliminaries}

\subsection{Proof of Lemma~\ref{lem:factor-oracle-variance}}

\begin{proof}
Fix \(V=\begin{bmatrix}B\\A^\top\end{bmatrix}\).
By the definition of \(G(V;\xi)\) and
Assumption~\ref{ass:finite-variance}, we have
\begin{equation*}
  \E_{\xi}[G(V;\xi)]
  =
  \begin{bmatrix}
    \E_{\xi}[\nabla F(BA;\xi)]A^\top\\
    \E_{\xi}[\nabla F(BA;\xi)]^\top B
  \end{bmatrix}
  =
  \begin{bmatrix}
    \nabla f(BA)A^\top\\
    \nabla f(BA)^\top B
  \end{bmatrix}
  =\nabla J(V).
\end{equation*}
Moreover, submultiplicativity of the Frobenius norm gives
\begin{align*}
  &\norm{G(V;\xi)-\nabla J(V)}^2\\
  &=\norm{(\nabla F(BA;\xi)-\nabla f(BA))A^\top}^2
    +\norm{(\nabla F(BA;\xi)-\nabla f(BA))^\top B}^2\\
  &\le\norm{\nabla F(BA;\xi)-\nabla f(BA)}^2
    \bigl(\norm A^2+\norm B^2\bigr).
\end{align*}
Taking expectation and applying Assumption~\ref{ass:finite-variance}
at \(W=BA\), we obtain
\begin{align*}
  &\E_{\xi}\norm{G(V;\xi)-\nabla J(V)}^2\\
  &\le\bigl(\norm A^2+\norm B^2\bigr)
    \E_{\xi}\norm{\nabla F(BA;\xi)-\nabla f(BA)}^2\\
  &\le\sigma^2\bigl(\norm A^2+\norm B^2\bigr)
    =\sigma^2\norm V^2.
\end{align*}
\end{proof}

\section{Proof of Section~\ref{sec:deterministic}}
\label{app:deterministic}

\subsection{Proof of Theorem~\ref{thm:deterministic}}

\begin{proof}
Set \(g_t=\norm{\nabla J(V_t)}\) and
\(\Delta=J(V_0)-f_\star\).  The stepsize in Algorithm~\ref{alg:lora-gd}
satisfies the hypothesis of Lemma~\ref{lem:appendix-one-step}.
Applying the lemma at each iteration and summing gives
\begin{equation}
  \sum_{t=0}^{T-1}\eta_tg_t^2\le4\Delta.
  \label{eq:appendix-budget}
\end{equation}
Thus \(J(V_t)\le J(V_0)\), and
Lemma~\ref{lem:appendix-gradient-bound}, applied at \(B_tA_t\), yields
\begin{equation}
  \norm{\nabla f(B_tA_t)}\le\sqrt{2\rho\Delta}.
  \label{eq:appendix-product-gradient-bound}
\end{equation}

The definition of \(\eta_t\) gives
\begin{equation}
  \eta_t\norm{V_t}^2\le\frac{1}{4\sqrt2\,\rho}
  \quad\text{and}\quad
  \eta_t\le1.
  \label{eq:appendix-step-controls}
\end{equation}
The update and Cauchy--Schwarz yield
\begin{align*}
  &\norm{V_{t+1}}^2
  \le\norm{V_t}^2+2\eta_t\norm{V_t}g_t+\eta_t^2g_t^2\\
  &\le\norm{V_t}^2+\frac{2\sqrt{\eta_t}g_t}{\sqrt{4\sqrt2\,\rho}}
    +\eta_tg_t^2.
\end{align*}
Summing over \(s=0,\ldots,t-1\), applying Cauchy--Schwarz, and using
\eqref{eq:appendix-budget} yield
\begin{align}
  &\norm{V_t}^2
  \le\norm{V_0}^2
  +\frac{2}{\sqrt{4\sqrt2\,\rho}}
    \sum_{s=0}^{t-1}\sqrt{\eta_s}g_s
  +\sum_{s=0}^{t-1}\eta_sg_s^2\nonumber\\
  &\le\norm{V_0}^2
  +2\sqrt{\frac{t}{4\sqrt2\,\rho}}
    \left(\sum_{s=0}^{t-1}\eta_sg_s^2\right)^{1/2}
  +\sum_{s=0}^{t-1}\eta_sg_s^2\nonumber\\
  &\le\norm{V_0}^2+4\sqrt{\frac{\Delta}{4\sqrt2\,\rho}}\,\sqrt t+4\Delta.
  \label{eq:appendix-factor-growth}
\end{align}

Define
\[
  C_0=\norm{V_0}^2+4\Delta+\sqrt{2\rho\Delta}
  \quad\text{and}\quad
  C_1=4\sqrt{\frac{\Delta}{4\sqrt2\,\rho}}.
\]
Equations \eqref{eq:appendix-product-gradient-bound} and
\eqref{eq:appendix-factor-growth} give
\[
  d_t=\norm{V_t}^2+\norm{\nabla f(B_tA_t)}
  \le C_0+C_1\sqrt t.
\]
Writing the stepsize in Algorithm~\ref{alg:lora-gd} as the reciprocal of
a maximum and substituting this bound, we obtain
\begin{equation}
  \begin{aligned}
    &\eta_t=\frac{1}{\max\{1,4\sqrt2\,\rho d_t\}}\\
    &\ge\frac{1}{\max\{1,4\sqrt2\,\rho(C_0+C_1\sqrt t)\}}\\
    &\ge\frac{1}{1+4\sqrt2\,\rho(C_0+C_1\sqrt t)}.
  \end{aligned}
  \label{eq:appendix-step-lower}
\end{equation}
Since \(\sqrt t\le\sqrt T\) for \(0\le t<T\), we have
\begin{equation}
  \sum_{t=0}^{T-1}\eta_t
  \ge\frac{T}{1+4\sqrt2\,\rho(C_0+C_1\sqrt T)}.
  \label{eq:appendix-step-sum}
\end{equation}
Equation~\eqref{eq:appendix-budget} implies
\[
  \min_{0\le t<T}g_t^2
  \le\frac{4\Delta}{\sum_{t=0}^{T-1}\eta_t}.
\]
Substituting \eqref{eq:appendix-step-sum} gives
\begin{equation}
  \min_{0\le t<T}g_t^2
  \le
  \frac{4\Delta\left[1+4\sqrt2\,\rho(C_0+C_1\sqrt T)\right]}{T}.
  \label{eq:appendix-deterministic-explicit}
\end{equation}
Set
\begin{equation}
\begin{split}
  C:=\max\Bigg\{1,\;16\Delta^2\Bigg[1+4\sqrt2\,\rho\Bigg(
  \norm{V_0}^2+4\Delta+\sqrt{2\rho\Delta}+4\sqrt{\frac{\Delta}{4\sqrt2\,\rho}}\Bigg)\Bigg]^2\Bigg\}.
\end{split}
\label{eq:appendix-deterministic-constant}
\end{equation}
For \(T\ge1\), the right side of
\eqref{eq:appendix-deterministic-explicit} is at most
\[
  4\Delta\left[1+4\sqrt2\,\rho(C_0+C_1)\right]T^{-1/2}
  \le \sqrt C\,T^{-1/2}.
\]
Thus, choosing \(T=\lceil C\epsilon^{-4}\rceil\) gives
\(\norm{\nabla J(V_{\widehat t})}\le\epsilon\).
\end{proof}

\section{Proof of Section~\ref{sec:stochastic}}
\label{app:stochastic}
\subsection{Proof of Proposition~\ref{prop:plain-lora-sgd-hard-instance}}

\begin{proof}
Let \(m=n=r=1\), write the factors as \(b,a\in\R\), and take
\[
  F(w;\xi)=\frac12w^2+\xi w,
  \qquad
  f(w)=\E_\xi[F(w;\xi)]=\frac12w^2,
  \qquad\text{and}\qquad
  \nabla F(w;\xi)=w+\xi.
\]
Let \(\xi\) have the continuous symmetric density
\[
  p(z)=\frac{3}{2(1+|z|)^4},
  \qquad z\in\R.
\]
Direct integration gives
\[
  \E[\xi]=0,
  \qquad
  \E[\xi^2]=1,
  \qquad\text{and}\qquad
  \E[|\xi|^q]=+\infty\quad\text{for every }q\ge3.
\]
Thus Assumption~\ref{ass:smooth} holds with \(\rho=1\),
Assumption~\ref{ass:lower-bounded} holds with \(f_\star=0\), and
Assumption~\ref{ass:finite-variance} holds with \(\sigma^2=1\).

Initialize \(b_0=a_0=1\).  Equality of the two factors is preserved by
\eqref{eq:plain-lora-sgd}; writing \(b_t=a_t=x_t\) gives
\begin{equation}
  x_{t+1}=x_t\bigl[1-\eta_t(x_t^2+\xi_t)\bigr].
  \label{eq:plain-sgd-scalar-recurrence}
\end{equation}
For fixed \(x_t\ne0\) and \(\eta_t>0\), continuity of the density gives
\[
  \Pr_{\xi_t}(x_{t+1}=0)
  =\Pr_{\xi_t}(\xi_t=\eta_t^{-1}-x_t^2)=0.
\]
Thus \(x_t\ne0\) almost surely for every \(t\), by induction from
\(x_0=1\).

Under the same conditioning, \(x_t\) and \(\eta_t\) are fixed, and
\eqref{eq:plain-sgd-scalar-recurrence} becomes
\[
  x_{t+1}=c_t+d_t\xi_t,
  \qquad
  c_t:=x_t(1-\eta_tx_t^2),
  \qquad\text{and}\qquad
  d_t:=-\eta_tx_t.
\]
The coefficient \(d_t\) is nonzero almost surely.  Set
\(R_t=2|c_t|/|d_t|<\infty\).  Whenever \(|\xi_t|\ge R_t\), the reverse
triangle inequality gives
\[
  |x_{t+1}|
  =|c_t+d_t\xi_t|
  \ge |d_t||\xi_t|-|c_t|
  \ge \frac{|d_t|}{2}|\xi_t|.
\]
Therefore, conditioning on the past gives
\[
  \E_{\xi_t}[|x_{t+1}|^3]
  \ge \frac{|d_t|^3}{8}
    \E_{\xi_t}\!\left[
      |\xi_t|^3\mathbf 1\{|\xi_t|\ge R_t\}
    \right]
  =+\infty.
\]
The last equality follows because removing any bounded interval from a
random variable with \(\E|\xi_t|^3=+\infty\) leaves an infinite third-moment
tail.  Since this conditional expectation is infinite almost surely,
averaging over the past gives \(\E|x_{t+1}|^3=+\infty\).  The same argument
with fourth powers gives \(\E|x_{t+1}|^4=+\infty\).
Finally, we have
\[
  J(x_t,x_t)=\frac12x_t^4,
  \qquad
  \nabla J(x_t,x_t)=\begin{bmatrix}x_t^3\\x_t^3\end{bmatrix},
  \qquad\text{and}\qquad
  \norm{\nabla J(x_t,x_t)}=\sqrt2\,|x_t|^3.
\]
Applying these identities at \(t+1\) proves
\eqref{eq:plain-sgd-hard-conclusion} for every \(t\ge1\).
\end{proof}

\subsection{Proof of Lemma~\ref{lem:nsgdm-local-control}}

\begin{proof}
Every update in Algorithm~\ref{alg:lora-nsgdm} has length at most
\(\gamma\).  The triangle inequality therefore gives
\[
  \norm{V_t}\le\norm{V_0}+t\gamma\le R_T.
\]
For every \(0\le t<T\), each point on the line segment from \(V_t\)
to \(V_{t+1}\) has the form \((1-\lambda)V_t+\lambda V_{t+1}\),
where \(\lambda\in[0,1]\).  By the triangle inequality, we have
\[
  \norm{(1-\lambda)V_t+\lambda V_{t+1}}
  \le(1-\lambda)\norm{V_t}+\lambda\norm{V_{t+1}}
  \le R_T.
\]
Thus the entire line segment lies in \(\Bcal_{R_T}\).

Let
\[
  V=\begin{bmatrix}B\\A^\top\end{bmatrix}
  \quad\text{and}\quad
  U=\begin{bmatrix}\widetilde B\\\widetilde A^\top\end{bmatrix}
\]
be in \(\Bcal_{R_T}\).  We have
\[
  BA-\widetilde B\widetilde A
  =\frac12\bigl(
    (B-\widetilde B)(A+\widetilde A)
    +(B+\widetilde B)(A-\widetilde A)
  \bigr).
\]
Submultiplicativity and Cauchy--Schwarz therefore give
\begin{equation}
  \norm{BA-\widetilde B\widetilde A}
  \le\frac12\norm{V-U}\norm{V+U}
  \le R_T\norm{V-U}.
  \label{eq:local-product-difference}
\end{equation}
Proposition~\ref{prop:factor-geometry} also gives
\(\norm{BA}\le\norm V^2/2\).  Hence, by
Assumption~\ref{ass:smooth}, we have
\[
  \sup_{V\in\Bcal_{R_T}}\norm{\nabla f(BA)}
  \le\norm{\nabla f(0)}+\frac{\rho R_T^2}{2}.
\]
Applying the triangle inequality and submultiplicativity to the
factor-gradient formula \eqref{eq:factor-gradient}, and then using
\eqref{eq:local-product-difference} and the preceding bound, we obtain
\begin{align*}
  &\norm{\nabla J(V)-\nabla J(U)}\\
  &\le\norm{\nabla f(BA)-\nabla f(\widetilde B\widetilde A)}\norm V
    +\norm{\nabla f(\widetilde B\widetilde A)}\norm{V-U}\\
  &\le\left(
    \norm{\nabla f(0)}+\frac{3\rho R_T^2}{2}
  \right)\norm{V-U},
\end{align*}
which proves \eqref{eq:nsgdm-local-smoothness}.

Finally, Lemma~\ref{lem:factor-oracle-variance} and
\(\norm{V_t}\le R_T\) give
\[
  \E_{\xi_t}[G(V_t;\xi_t)-\nabla J(V_t)]=0
  \quad\text{and}\quad
  \E_{\xi_t}\norm{G(V_t;\xi_t)-\nabla J(V_t)}^2
  \le\sigma^2\norm{V_t}^2
  \le\sigma^2R_T^2,
\]
which proves \eqref{eq:nsgdm-noise-scale}.
\end{proof}

\subsection{Proof of Lemma~\ref{lem:nsgdm-tracking}}

\begin{proof}
Let
\[
  G_t:=\nabla J(V_t),
  \qquad
  Z_t:=G(V_t;\xi_t)-G_t,
  \qquad\text{and}\qquad
  E_t:=M_t-G_t,
\]
and let \(q:=1-\alpha\).  Lemma~\ref{lem:nsgdm-local-control} gives
\[
  \E_{\xi_t}[Z_t]=0
  \quad\text{and}\quad
  \E_{\xi_t}\norm{Z_t}^2\le\sigma^2R_T^2.
\]
For \(t\ge1\), we have
\[
  E_t=qE_{t-1}+q(G_{t-1}-G_t)+\alpha Z_t
  \quad\text{and}\quad
  E_0=-qG_0+\alpha Z_0.
\]
Unrolling gives
\begin{equation}
  E_t=-q^{t+1}G_0
  +q\sum_{s=1}^tq^{t-s}(G_{s-1}-G_s)
  +\alpha\sum_{s=0}^tq^{t-s}Z_s.
  \label{eq:appendix-error-unroll}
\end{equation}
Local smoothness and the step-length bound imply
\(\norm{G_s-G_{s-1}}\le L_T\gamma\).  Hence the drift term in
\eqref{eq:appendix-error-unroll} satisfies
\[
  \norm{q\sum_{s=1}^tq^{t-s}(G_{s-1}-G_s)}
  \le qL_T\gamma\sum_{k=0}^{t-1}q^k
  \le\frac{L_T\gamma}{\alpha},
\]
since \(1-q=\alpha\).

For \(r<s\), conditional unbiasedness gives
\[
  \E\inner{Z_r}{Z_s}
  =\E_{\xi_0,\ldots,\xi_{s-1}}
    \inner{Z_r}{\E_{\xi_s}[Z_s]}
  =0.
\]
Consequently, we have
\begin{align*}
  &\E\norm{\alpha\textstyle\sum_{s=0}^tq^{t-s}Z_s}
  \le\left(
    \E\norm{\alpha\textstyle\sum_{s=0}^tq^{t-s}Z_s}^2
  \right)^{1/2}\\
  &=\alpha\left(
    \sum_{s=0}^tq^{2(t-s)}\E\norm{Z_s}^2
  \right)^{1/2}\\
  &\le\alpha\sigma R_T
    \left(\sum_{k=0}^{t}q^{2k}\right)^{1/2}\\
  &\le \sigma R_T\sqrt{\frac{\alpha}{2-\alpha}}
  \le \sigma R_T\sqrt\alpha.
\end{align*}
Combining the preceding two estimates with
\eqref{eq:appendix-error-unroll} and the triangle inequality gives, for
every \(0\le t<T\),
\[
  \E\norm{E_t}
  \le q^{t+1}\norm{G_0}
  +\frac{L_T\gamma}{\alpha}
  +\sigma R_T\sqrt\alpha.
\]
Since \(1-q=\alpha\), the initialization weights obey
\[
  \sum_{t=0}^{T-1}q^{t+1}
  =\frac{q(1-q^T)}{1-q}
  \le\frac{1}{\alpha}.
\]
Therefore, averaging the pointwise bound over \(t=0,\ldots,T-1\) yields
\begin{align*}
  &\frac1T\sum_{t=0}^{T-1}\E\norm{E_t}
  \le
  \frac{\norm{G_0}}{T}\sum_{t=0}^{T-1}q^{t+1}
  +\frac{L_T\gamma}{\alpha}
  +\sigma R_T\sqrt\alpha\\
  &\le
  \frac{\norm{G_0}}{\alpha T}
  +\frac{L_T\gamma}{\alpha}
  +\sigma R_T\sqrt\alpha.
\end{align*}
Substituting \(E_t=M_t-\nabla J(V_t)\) and
\(G_0=\nabla J(V_0)\) gives \eqref{eq:main-tracking}.
\end{proof}

\subsection{Proof of Lemma~\ref{lem:nsgdm-descent}}

\begin{proof}
Let
\[
  G_t:=\nabla J(V_t)
  \quad\text{and}\quad
  E_t:=M_t-G_t.
\]
If \(M_t\ne0\), then we have
\begin{equation}
  \begin{aligned}
    &\inner{G_t}{M_t/\norm{M_t}}
    =\inner{M_t-E_t}{M_t/\norm{M_t}}\\
    &=\norm{M_t}-\inner{E_t}{M_t/\norm{M_t}}\\
    &\ge\norm{M_t}-\norm{E_t}\\
    &\ge\norm{G_t}-2\norm{E_t}.
  \end{aligned}
  \label{eq:appendix-alignment}
\end{equation}
The first inequality follows from Cauchy--Schwarz, and the second follows
from \(\norm{M_t}\ge\norm{G_t}-\norm{E_t}\).

Every update segment is contained in \(\Bcal_{R_T}\).  The local descent
lemma, the update in Algorithm~\ref{alg:lora-nsgdm}, and
\eqref{eq:appendix-alignment} therefore give, whenever \(M_t\ne0\),
\begin{align*}
  &J(V_{t+1})
  \le J(V_t)
    +\inner{G_t}{V_{t+1}-V_t}
    +\frac{L_T}{2}\norm{V_{t+1}-V_t}^2\\
  &=J(V_t)-\gamma\inner{G_t}{M_t/\norm{M_t}}
    +\frac{L_T\gamma^2}{2}\\
  &\le J(V_t)-\gamma\norm{G_t}
    +2\gamma\norm{E_t}+\frac{L_T\gamma^2}{2}.
\end{align*}
If \(M_t=0\), then \(E_t=-G_t\), and the zero update satisfies the same
final inequality.  Substituting the definitions of \(G_t\) and \(E_t\)
gives \eqref{eq:nsgdm-descent}.
\end{proof}

\subsection{Proof of Theorem~\ref{thm:stochastic}}

\begin{proof}
Let
\[
  G_t:=\nabla J(V_t).
\]
Summing \eqref{eq:nsgdm-descent} over \(t=0,\ldots,T-1\), taking
expectation, using \(J(V_T)\ge f_\star\), and applying
Lemma~\ref{lem:nsgdm-tracking} yield
\begin{equation}
  \frac1T\sum_{t<T}\E\norm{G_t}
  \le
  \frac{\Delta}{\gamma T}
  +\frac{2\norm{G_0}}{\alpha T}
  +\frac{2L_T\gamma}{\alpha}
  +2\sigma R_T\sqrt\alpha
  +\frac{L_T\gamma}{2}.
  \label{eq:main-stochastic-bound}
\end{equation}
Let
\begin{align*}
  &\overline R=\norm{V_0}+1,
  \qquad
  \overline L=\norm{\nabla f(0)}+\frac{3\rho\overline R^2}{2},\\
  &\text{and}\quad
  G_{\rm init}
  =\norm{V_0}\left(\norm{\nabla f(0)}
    +\frac{\rho}{2}\norm{V_0}^2\right).
\end{align*}
Proposition~\ref{prop:factor-geometry} and the \(\rho\)-smoothness of \(f\)
give \(\norm{G_0}\le G_{\rm init}\).
With \(\gamma=T^{-7/8}\), we have, for \(T\ge1\),
\[
  R_T\le\overline R T^{1/8},
  \qquad
  L_T\le\overline L T^{1/4},
  \qquad\text{and}\qquad
  \sigma R_T\le\sigma\overline R T^{1/8}.
\]
Substituting these inequalities and \(\alpha=T^{-1/2}\) into
\eqref{eq:main-stochastic-bound} yields
\begin{equation}
  \frac1T\sum_{t=0}^{T-1}\E\norm{G_t}
  \le
  \left(\Delta+2G_{\rm init}+2\sigma\overline R
    +\frac52\overline L\right)T^{-1/8}.
  \label{eq:appendix-explicit-stochastic}
\end{equation}
For an independent uniform \(I\), the left side is
\(\E\norm{\nabla J(V_I)}\).  We take
\(C_{\rm mom}=\Delta+2G_{\rm init}+2\sigma\overline R
+5\overline L/2\).
In particular, we have
\[
  C_{\rm mom}
  =\mathcal{O}\!\left(
    1+\Delta
    +(1+\norm{V_0})\bigl(\norm{\nabla f(0)}+\sigma\bigr)
    +\rho\bigl(\norm{V_0}^3+(1+\norm{V_0})^2\bigr)
  \right).
\]
The choice of \(T\) in \eqref{eq:stochastic-complexity} therefore makes
this expectation at most \(\epsilon\).  Since
Algorithm~\ref{alg:lora-nsgdm} makes one stochastic oracle evaluation per
iteration, the same bound gives its stochastic oracle complexity.
\end{proof}

\section{Proof of Section~\ref{sec:variance-reduction}}
\label{app:variance-reduced}

\subsection{Proof of Proposition~\ref{prop:factor-mss}}

\begin{proof}
First, unbiasedness and Jensen's inequality imply that, for any
\(W,W'\in\R^{m\times n}\), we have
\begin{align*}
  &\norm{\nabla f(W)-\nabla f(W')}^2
    =\norm{\E_{\xi}[\nabla F(W;\xi)-\nabla F(W';\xi)]}^2\\
  &\le\E_{\xi}\norm{\nabla F(W;\xi)-\nabla F(W';\xi)}^2
    \le\ell^2\norm{W-W'}^2,
\end{align*}
where the last inequality follows from Assumption~\ref{ass:mss}.
Thus \(\nabla f\) is \(\ell\)-Lipschitz continuous.

Let
\[
  V=\begin{bmatrix}B\\A^\top\end{bmatrix}
  \quad\text{and}\quad
  U=\begin{bmatrix}\widetilde B\\\widetilde A^\top\end{bmatrix},
\]
and set \(R=\max\{\norm U,\norm V\}\).
The argument for \eqref{eq:local-product-difference}, with radius \(R\), gives
\[
  \norm{BA-\widetilde B\widetilde A}
  \le\frac12\norm{V-U}\norm{V+U}
  \le R\norm{V-U}.
\]
By the definition of \(G\), submultiplicativity, and
\(\norm{X+Y}^2\le2\norm X^2+2\norm Y^2\), we have
\begin{align*}
  &\E_{\xi}\norm{G(V;\xi)-G(U;\xi)}^2\\
  &\le2\norm V^2\,
    \E_{\xi}\norm{\nabla F(BA;\xi)-\nabla F(\widetilde B\widetilde A;\xi)}^2
    +2\norm{V-U}^2\,
      \E_{\xi}\norm{\nabla F(\widetilde B\widetilde A;\xi)}^2.
\end{align*}
Assumption~\ref{ass:mss} and the product-difference bound imply
\[
  \E_{\xi}\norm{\nabla F(BA;\xi)-\nabla F(\widetilde B\widetilde A;\xi)}^2
  \le\ell^2\norm{BA-\widetilde B\widetilde A}^2
  \le\ell^2R^2\norm{V-U}^2.
\]
For the remaining second moment, Assumption~\ref{ass:finite-variance}
and the \(\ell\)-Lipschitz continuity of \(\nabla f\) give
\begin{align*}
  &\E_{\xi}\norm{\nabla F(\widetilde B\widetilde A;\xi)}^2\\
  &\le\norm{\nabla f(\widetilde B\widetilde A)}^2+\sigma^2\\
  &\le\left(\norm{\nabla f(0)}
    +\ell\norm{\widetilde B\widetilde A}\right)^2+\sigma^2\\
  &\le\left(\norm{\nabla f(0)}
    +\frac{\ell R^2}{2}\right)^2+\sigma^2,
\end{align*}
where the last inequality uses
\(\norm{\widetilde B\widetilde A}\le\norm U^2/2\).
Combining these estimates, we obtain
\begin{equation}
  \E_{\xi}\norm{G(V;\xi)-G(U;\xi)}^2
  \le
  2\left(
    \ell^2R^4+
    \left(\norm{\nabla f(0)}+\frac{\ell R^2}{2}\right)^2+
    \sigma^2
  \right)\norm{V-U}^2.
  \label{eq:factor-mss-intermediate}
\end{equation}
Applying \(2(a+b)^2\le4a^2+4b^2\) to
\eqref{eq:factor-mss-intermediate} gives
\[
  \E_{\xi}\norm{G(V;\xi)-G(U;\xi)}^2
  \le\bigl(3\ell^2R^4+4\norm{\nabla f(0)}^2+2\sigma^2\bigr)
    \norm{V-U}^2,
\]
which is \eqref{eq:factor-mss}.
\end{proof}

\subsection{Auxiliary tracking lemma}

\begin{lemma}
\label{lem:vr-tracking}
Suppose Assumptions~\ref{ass:finite-variance} and~\ref{ass:mss} hold, and
let the iterates be generated by Algorithm~\ref{alg:lora-storm}.  Set
\[
  R_T=\norm{V_0}+T\eta,
  \qquad S_T=\sigma R_T,
  \qquad
  S_0=\sigma\norm{V_0},
\]
and
\[
  \Lambda_T^2
  =2\ell^2R_T^4
   +2\left[\left(\norm{\nabla f(0)}+\frac{\ell R_T^2}{2}\right)^2
   +\sigma^2\right].
\]
Then the estimator satisfies
\begin{equation}
  \frac1T\sum_{t=0}^{T-1}
  \E\norm{D_t-\nabla J(V_t)}
  \le \frac{S_0}{aT}+\sqrt{2a}\,S_T
  +\frac{\Lambda_T\eta}{\sqrt a}.
  \label{eq:appendix-vr-tracking}
\end{equation}
\end{lemma}

\begin{proof}
Set
\[
  G_t:=\nabla J(V_t)
  \quad\text{and}\quad
  E_t:=D_t-G_t.
\]
By the initialization in Algorithm~\ref{alg:lora-storm}, we have
\[
  \E\norm{E_0}^2
  =\E\norm{G(V_0;\xi_0)-G_0}^2
  \le \sigma^2\norm{V_0}^2=S_0^2.
\]
For \(t\ge1\), set
\[
  \Xi_t=G(V_t;\xi_t)-G_t
  \quad\text{and}\quad
  \Xi_{t-1}^{(t)}
  =G(V_{t-1};\xi_t)-G_{t-1},
\]
and set \(\Theta_t=\Xi_t-\Xi_{t-1}^{(t)}\) and \(q=1-a\).
The estimator update in Algorithm~\ref{alg:lora-storm} gives
\[
  D_t=G(V_t;\xi_t)
  +q\bigl(D_{t-1}-G(V_{t-1};\xi_t)\bigr).
\]
Subtracting \(G_t\) and using \(a=1-q\) therefore yields
\begin{equation}
  E_t=qE_{t-1}+a\Xi_t+q\Theta_t.
  \label{eq:appendix-vr-error}
\end{equation}
Let
\[
  X_t:=G(V_t;\xi_t)-G(V_{t-1};\xi_t).
\]
Lemma~\ref{lem:factor-oracle-variance} gives
\[
  \E_{\xi_t}[\Xi_t]=0
  \quad\text{and}\quad
  \E_{\xi_t}\norm{\Xi_t}^2\le S_T^2.
\]
Moreover, we have
\[
  \E_{\xi_t}[X_t]=G_t-G_{t-1}
  \quad\text{and}\quad
  \Theta_t=X_t-\E_{\xi_t}[X_t].
\]
Hence, by \eqref{eq:factor-mss-intermediate} and
\(\norm{V_t-V_{t-1}}\le\eta\), we have $\E_{\xi_t}[\Theta_t]=0$ and 
\begin{align*}
  &\E_{\xi_t}\norm{\Theta_t}^2
  =\E_{\xi_t}\norm{X_t-\E_{\xi_t}[X_t]}^2\\
  &\le\E_{\xi_t}\norm{X_t}^2\le\Lambda_T^2\norm{V_t-V_{t-1}}^2\le\Lambda_T^2\eta^2.
\end{align*}
The cross term in \eqref{eq:appendix-vr-error} satisfies
\[
  \E_{\xi_t}\inner{E_{t-1}}{a\Xi_t+q\Theta_t}
  =\inner{E_{t-1}}
    {a\E_{\xi_t}[\Xi_t]+q\E_{\xi_t}[\Theta_t]}
  =0.
\]
Taking expectation and using
\(\norm{X+Y}^2\le2\norm X^2+2\norm Y^2\), we obtain
\begin{align*}
  &\E\norm{E_t}^2
  =q^2\E\norm{E_{t-1}}^2
    +\E\norm{a\Xi_t+q\Theta_t}^2\\
  &\le q^2\E\norm{E_{t-1}}^2
    +2a^2\E\norm{\Xi_t}^2
    +2q^2\E\norm{\Theta_t}^2\\
  &\le q^2\E\norm{E_{t-1}}^2
    +2a^2S_T^2+2q^2\Lambda_T^2\eta^2.
\end{align*}
Since \(1-q^2=a(2-a)\) and \(2q^2\le2-a\) for \(a\in(0,1]\),
unrolling this recurrence gives
\begin{align*}
  &\E\norm{E_t}^2
  \le q^{2t}S_0^2
    +\left(2a^2S_T^2+2q^2\Lambda_T^2\eta^2\right)
      \sum_{k=0}^{t-1}q^{2k}\\
  &\le q^{2t}S_0^2+2aS_T^2
    +\frac{\Lambda_T^2\eta^2}{a}.
\end{align*}
Jensen's inequality and
\(\sqrt{x+y+z}\le\sqrt x+\sqrt y+\sqrt z\) give
\[
  \E\norm{E_t}
  \le q^tS_0+\sqrt{2a}\,S_T+\frac{\Lambda_T\eta}{\sqrt a}.
\]
Averaging over \(t=0,\ldots,T-1\) gives
\begin{align*}
  &\frac1T\sum_{t=0}^{T-1}\E\norm{D_t-\nabla J(V_t)}
  =\frac1T\sum_{t=0}^{T-1}\E\norm{E_t}\\
  &\le\frac{S_0}{T}\sum_{t=0}^{T-1}q^t
    +\sqrt{2a}\,S_T+\frac{\Lambda_T\eta}{\sqrt a}\\
  &\le\frac{S_0}{aT}
    +\sqrt{2a}\,S_T+\frac{\Lambda_T\eta}{\sqrt a}.
\end{align*}
\end{proof}

\subsection{Proof of Theorem~\ref{thm:variance-reduced}}

\begin{proof}
Let \(R_T\), \(S_T\), \(S_0\), and \(\Lambda_T\) be as in
Lemma~\ref{lem:vr-tracking}, and set
\[
  L_T=\norm{\nabla f(0)}+\frac{3\ell R_T^2}{2}.
\]
The normalized update in Algorithm~\ref{alg:lora-storm} gives
\(\norm{V_{t+1}-V_t}\le\eta\), so all iterates and update segments lie
in \(\Bcal_{R_T}\).  Assumptions~\ref{ass:finite-variance}
and~\ref{ass:mss} imply that \(\nabla f\) is \(\ell\)-Lipschitz
continuous.  The local smoothness argument in the proof of
Lemma~\ref{lem:nsgdm-local-control}, with \(\rho=\ell\), therefore shows
that \(J\) is \(L_T\)-smooth on this ball.

The descent argument in the proof of Lemma~\ref{lem:nsgdm-descent},
with \(D_t\) and \(\eta\) in place of \(M_t\) and \(\gamma\), gives
\[
  J(V_{t+1})
  \le J(V_t)-\eta\norm{\nabla J(V_t)}
    +2\eta\norm{D_t-\nabla J(V_t)}+\frac{L_T\eta^2}{2}.
\]
Summing over \(t=0,\ldots,T-1\), taking expectation, and using
\(J(V_T)\ge f_\star\), we obtain
\begin{equation}
  \begin{aligned}
    &\frac1T\sum_{t=0}^{T-1}\E\norm{\nabla J(V_t)}
    \le\frac{\Delta}{\eta T}
      +\frac2T\sum_{t=0}^{T-1}\E\norm{D_t-\nabla J(V_t)}
      +\frac{L_T\eta}{2}\\
    &\le\frac{\Delta}{\eta T}+\frac{2S_0}{aT}
      +2\sqrt{2a}\,S_T+\frac{2\Lambda_T\eta}{\sqrt a}
      +\frac{L_T\eta}{2},
  \end{aligned}
  \label{eq:appendix-vr-main-bound}
\end{equation}
where the last inequality follows from Lemma~\ref{lem:vr-tracking}.

Algorithm~\ref{alg:lora-storm} sets \(a=T^{-2/3}\) and
\(\eta=T^{-5/6}\).  Since \(T\ge1\), the radius satisfies
\[
  R_T=\norm{V_0}+T^{1/6}
  \le(1+\norm{V_0})T^{1/6}.
\]
Using \((x+y)^2\le2x^2+2y^2\) in the definition of \(\Lambda_T\),
we further obtain
\begin{align*}
  &\Lambda_T
  \le\bigl(3\ell^2R_T^4+4\norm{\nabla f(0)}^2+2\sigma^2\bigr)^{1/2}\\
  &\le\sqrt3\,\ell R_T^2+2\norm{\nabla f(0)}+\sqrt2\,\sigma\\
  &\le\bigl(\sqrt3\,\ell(1+\norm{V_0})^2
    +2\norm{\nabla f(0)}+\sqrt2\,\sigma\bigr)T^{1/3}.
\end{align*}
Substituting the choices of \(a\) and \(\eta\) into
\eqref{eq:appendix-vr-main-bound}, and then applying these bounds and
the definitions of \(S_0\), \(S_T\), and \(L_T\), gives
\begin{equation}
  \begin{aligned}
    &\frac1T\sum_{t=0}^{T-1}\E\norm{\nabla J(V_t)}\\
    &\le\Delta T^{-1/6}+2\sigma\norm{V_0}T^{-1/3}
      +2\sqrt2\,\sigma R_TT^{-1/3}+2\Lambda_TT^{-1/2}\\
    &\qquad+\biggl(\frac{\norm{\nabla f(0)}}2
      +\frac{3\ell R_T^2}{4}\biggr)T^{-5/6}\\
    &\le\Delta T^{-1/6}+2\sigma\norm{V_0}T^{-1/3}
      +2\sqrt2\,\sigma(1+\norm{V_0})T^{-1/6}\\
    &\qquad+2\bigl(\sqrt3\,\ell(1+\norm{V_0})^2
      +2\norm{\nabla f(0)}+\sqrt2\,\sigma\bigr)T^{-1/6}\\
    &\qquad+\biggl(\frac{\norm{\nabla f(0)}}2
      +\frac{3\ell(1+\norm{V_0})^2}{4}\biggr)T^{-1/2}\\
    &\le8\bigl(1+\Delta
      +(1+\norm{V_0})(\norm{\nabla f(0)}+\sigma)
      +\ell(1+\norm{V_0})^2\bigr)T^{-1/6}.
  \end{aligned}
  \label{eq:appendix-vr-explicit}
\end{equation}
The last inequality uses \(T^{-1/3}\le T^{-1/6}\) and
\(T^{-1/2}\le T^{-1/6}\).

Since \(I\) is independent of the oracle samples and uniform on
\(\{0,\ldots,T-1\}\), we have
\[
  \E\norm{\nabla J(V_I)}
  =\frac1T\sum_{t=0}^{T-1}\E\norm{\nabla J(V_t)}.
\]
Thus, choosing \(T\) as in \eqref{eq:variance-reduced-complexity}, with a
sufficiently large universal constant, makes this expectation at most
\(\epsilon\).  Algorithm~\ref{alg:lora-storm} uses one stochastic oracle
evaluation at initialization and two per subsequent iteration, for a
total of \(2T-1\) evaluations.  Its stochastic oracle complexity therefore
has the same order as \(T\).
\end{proof}

\section{Experimental Details}\label{app:experimental-details} 
 This appendix specifies the experimental protocol used in
Section~\ref{sec:experiments}. Hyperparameters are selected on validation data;
the test splits are not used. Unless explicitly noted, the tasks, architectures,
and LoRA configurations follow \citet{mu2026loraGD}. We fix each task and model
configuration before tuning the optimizers.

\subsection{Task and model configurations}

Table~\ref{tab:experiment-hyperparameters} summarizes the principal
hyperparameters for the four reported settings. Our \(c_{\mathrm{adapt}}\), \(c_{\mathrm{adapt2}}\), and \(c_{\mathrm{norm}}\) correspond to \(\alpha^{\mathrm{adapt}}\), \(\alpha^{\mathrm{adapt2}}\), and \(\alpha^{\mathrm{norm}}\), respectively, in \citet{mu2026loraGD}. Panels (a) and (b) of
Figure~\ref{fig:feature-convergence} use 60 epochs. Figure~\ref{fig:resnet-convergence} reports
250 epochs of ResNet-18 training. The small- and large-initialization panels of
Figure~\ref{fig:tinyllama-initialization} use 500 and 250 steps, respectively.
Each optimizer iteration uses one minibatch. The first LoRA-STORM update uses one factor-gradient evaluation. Each subsequent update uses two evaluations on the same fresh minibatch, one at the current factors and one at the preceding factors. Thus, \(T\) updates require \(2T-1\) stochastic-oracle calls.

\begin{table}[t]
  \centering
  \caption{Experiment hyperparameters.}
  \label{tab:experiment-hyperparameters}
  \small
  \setlength{\tabcolsep}{4pt}
  \begin{tabular}{@{}lcccc@{}}
    \toprule
    Hyperparameter & Logistic & ResNet-18 & TinyLlama small & TinyLlama large \\
    \midrule
    Batch size & 512 & 512 & 16 & 16 \\
    LoRA rank & 4 & 20 & 32 & 32 \\
    Input size & 512 & $32\times32\times3$ & 512 tokens & 512 tokens \\
    Training horizon & 60/1,000 epochs & 250 epochs & 500 steps & 250 steps \\
    \midrule
    $c_{\mathrm{adapt}}$ & 1.0 & -- & -- & -- \\
    $c_{\mathrm{adapt2}}$ & 0.8 & 15 & 0.2 & 562.9130541 \\
    $c_{\mathrm{norm}}$ & 0.06 & 0.1 & 0.02 & 0.07931165 \\
    \algname{} $(\alpha,\gamma)$ & $(0.2,0.1)$ & $(0.1,0.1)$ & $(0.2,0.1)$ & $(0.5,0.2)$ \\
    \vralgname{} $(a,\eta_0)$ & $(0.5,0.1)$ & -- & -- & -- \\
    \bottomrule
  \end{tabular}
\end{table}

\paragraph{Logistic regression on fixed ResNet-18 features.}
We use the 50,000-example torchvision CIFAR-10 training split and make a
deterministic 45,000/5,000 training/validation split. Images are resized to
\(224\times224\), normalized by the ImageNet mean and standard deviation, and
encoded once by the penultimate layer of torchvision's ImageNet-1K V1
ResNet-18. The resulting 512-dimensional tensors are cached and shared by
every method. A randomly initialized linear classifier weight is frozen and
augmented by the unscaled rank-4 update \(BA\). We initialize \(A\) with PyTorch's
Kaiming-uniform rule and \(B=0\), so every optimizer starts from the same
predictor. Training uses shuffled minibatches, cross-entropy, and no data
augmentation or weight decay. With 45,000 examples and batch size 512, an
epoch contains 88 updates; the final minibatch is smaller than 512.

\paragraph{CIFAR-10 classification with convolutional LoRA.}
We use the \texttt{uoft-cs/cifar10} Hugging Face mirror and initialize a
torchvision ResNet-18 from scratch following \citet{mu2026loraGD}. Rank-20 adapters are applied to  every \texttt{conv1} and \texttt{conv2} in the residual blocks, and the
convolutional downsample projections. The classification head is trained in
full rank. All remaining ResNet parameters, including BatchNorm affine
parameters, are frozen. The model remains in evaluation mode during
optimization. This freezes the BatchNorm running
statistics but leaves gradient computation enabled for the LoRA factors and
classifier. Because the model is initialized from scratch and we perform no
BatchNorm calibration pass, the stored running means and variances remain at
their initialization. This evaluation-mode protocol is applied identically to
all methods. Inputs are converted to tensors and normalized channelwise by
mean and standard deviation \(0.5\), without augmentation. The deterministic
45,000/5,000 split is shared by all methods, the official test split is
unused, and 250 epochs correspond to 22,000 updates.

\paragraph{TinyLlama instruction tuning.}
We use TinyLlama-1.1B-Chat-v1.0 and Alpaca. We reserve 2,000 examples for validation. Tuning evaluates a fixed 256-example
subset, while reported terminal validation uses all 2,000 examples. Each
example is rendered with the standard Alpaca instruction, optional input, and
response fields, then truncated or padded to 512 tokens. Padding uses the
end-of-sequence token and is masked from the causal-language-model loss.

We adapt only \texttt{q\_proj}, \texttt{k\_proj}, \texttt{v\_proj}, and
\texttt{o\_proj}. The rank and LoRA scaling parameter are both 32, so the
adapter multiplier is one; biases remain frozen. The configured LoRA dropout
is 0.05, but the model is kept in evaluation mode and hence dropout is
inactive. The \(A\) factor is initialized as a Gaussian random matrix with standard deviation \(\sigma=10^{-3}\) in the small initialization setting, or \(\sigma=1/r=0.03125\) in the large initialization setting. In both settings, we set \(B=0\) initially. Each update uses one minibatch of 16 examples.

\subsection{Baselines and selected hyperparameters}

For coefficient \(c\), the LoRA-GD stepsize rules used for LoRA-GD algorithms are
\begin{align*}
  \eta_t^{\mathrm{adapt}}
  &=\min\!\left\{\frac{c}{\norm{V_t}^2+\norm{h_t}},1\right\},\\
  \eta_t^{\mathrm{adapt2}}
  &=\min\!\left\{\frac{c}{\norm{V_t}^2+\sqrt{\ell_t}},1\right\},\\
  \eta_t^{\mathrm{norm}}
  &=\min\!\left\{\frac{c}{\sqrt{\norm{g_t}}},1\right\}.
\end{align*}
where \(V_t\) is the pair of LoRA factors and \(\ell_t\) is the current
minibatch loss. The quantities \(g_t\) and \(h_t\) are the gradients with
respect to the factors and the product \(BA\), respectively. The
\(\eta_t^{\mathrm{adapt}}\) rule is used only for logistic regression, as in
\citet{mu2026loraGD}. Consistent with Section~\ref{sec:experiments}, we also
evaluate \vralgname{} only for logistic regression because it uses two
stochastic-oracle evaluations per iteration. Table~\ref{tab:experiment-hyperparameters}
records every coefficient used in Figures~\ref{fig:feature-convergence}--\ref{fig:tinyllama-initialization}.
The fixed LoRA-GD learning rates
\(\{0.03,0.05,0.07\}\) for logistic regression and
\(\{0.03,0.04,0.045\}\) for convolutional ResNet-18 were reproduced and retained
in the artifact bundle, but are omitted from the main comparison for the
reason stated in Section~\ref{sec:experiments}.
Hyperparameters are selected on validation data, and are ranked by training-loss AUC unless stated otherwise. Terminal validation loss is used as the tie breaker.

\paragraph{LoRA-NSGDM calibration.}
For every task we use the same coarse grid
\[
  \alpha\in\{0.05,0.1,0.2,0.5,1.0\},\qquad
  \gamma\in\{0.01,0.02,0.05,0.1,0.2\}.
\]
For logistic regression, all 25 candidates are trained for 5,280 steps and
ranked by epoch-loss AUC, selecting \((0.2,0.1)\). For convolutional ResNet-18,
candidates are trained for 50 epochs. We minimize batch-loss AUC and select \((0.1,0.1)\). TinyLlama
uses a training budget of 500 steps. The selected
parameters are \((0.2,0.1)\) for small initialization and \((0.5,0.2)\) for
large initialization.

\paragraph{Large-initialization LoRA-GD calibration.} In the large-initialization setting, the hyperparameters for the LoRA-GD algorithms are not specified in the open-source code repository of \cite{mu2026loraGD}. Moreover, we cannot reuse the hyperparameters for small initialization because
large initialization substantially changes the factor norm in the denominator
of the LoRA-GD stepsize rules. We therefore tune a coarse grid of target initial
effective stepsizes
\(\eta_0\in\{0.05,0.10,0.20\}\). A fixed pilot minibatch gives
\(\ell_0=1.9944235\), \(\norm{V_0}=75.0181198\), and
\(\norm{g_0}=2.5161352\). By the stepsize rules, we have
\[
\begin{split}
 c_{\mathrm{adapt2}}&\in
 \{281.4565270,562.9130541,1125.8261081\},\\
 c_{\mathrm{norm}}&\in
 \{0.07931165,0.15862330,0.31724661\}.
\end{split}
\]
We select \(562.9130541\) and \(0.07931165\), respectively. Thus, the chosen
\(\eta_t^{\mathrm{adapt2}}\) rule targets pilot \(\eta_0=0.10\), while the
\(\eta_t^{\mathrm{norm}}\) rule targets pilot \(\eta_0=0.05\). The non-round coefficients follow
algebraically from the pilot quantities above.

\paragraph{LoRA-STORM calibration.}
We use the cosine-decayed normalized stepsize
\[
  \eta_t=\eta_0\left[r_{\min}
  +\frac{1-r_{\min}}{2}\left(1+\cos\!\left(\frac{\pi t}{T-1}\right)\right)\right],
  \qquad t=0,\ldots,T-1,
\]
which decreases from \(\eta_0\) to \(r_{\min}\eta_0\) over a run of
\(T\) steps. We fix \(r_{\min}=0.001\). The tuning grid is
\[
\begin{aligned}
  a&\in\{0.1, 0.2, 0.5, 1\},\\
  \eta_0&\in\{0.1,0.2,0.5, 1\}.
\end{aligned}
\]
We select \((a,\eta_0)=(0.5,0.1)\). Because the schedule depends on \(T\),
the 5,280- and 88,000-step trajectories are run separately using the same
selected hyperparameters.

\subsection{Evaluation and plotting}

The loss panels report pre-update minibatch training loss.
Figure~\ref{fig:feature-convergence} uses moving-average windows 10, 200, and
200 in panels (a)--(c), respectively. Figure~\ref{fig:resnet-convergence} uses
window 50 in all three panels; panels (b) and (c) report the actual factor-update
norm and minibatch-gradient norm. Figure~\ref{fig:tinyllama-initialization}
uses window 10 in both panels. Its curves prepend the common pre-update loss at
step zero.

\subsection{Hardware and software}

The logistic-regression experiments were run on an Apple M2 CPU under macOS 15.0.1. The convolutional ResNet-18 experiments
were run on an AMD EPYC 7542 server with one NVIDIA GeForce RTX 4090 GPU (24
GB). The TinyLlama experiments were run on an Intel Xeon Gold 6530 server with
two NVIDIA GeForce RTX 4090 GPUs (48 GB each). Both GPU servers use Ubuntu
22.04.5 and CUDA 13.0. Each training trajectory
uses one GPU; multiple GPUs are used only to execute independent
trajectories concurrently.

\end{document}